\documentclass{article}
\usepackage{amsthm}

\PassOptionsToPackage{numbers}{natbib}
\usepackage[nonatbib,final]{neurips_2021}

\RequirePackage[l2tabu, orthodox]{nag}
\usepackage{setspace}
\usepackage{makecell}
\newcolumntype{M}[1]{>{\centering\arraybackslash}m{#1}}
\usepackage[T1]{fontenc}

\usepackage{fouriernc}
\usepackage{amsmath}
\usepackage[scale=0.92]{tgschola}
\usepackage[scaled=0.88]{PTSans}
\usepackage{scalefnt,letltxmacro}
\LetLtxMacro{\oldtextsc}{\textsc}
\renewcommand{\textsc}[1]{\oldtextsc{\scalefont{1.25}#1}}
\usepackage{inconsolata}

\usepackage{amssymb}
\usepackage{mathbbol}
\usepackage{PTSans}
\usepackage{inconsolata}

\usepackage{xcolor}
\definecolor{shadecolor}{gray}{0.9}

\usepackage[english]{babel}
\usepackage{afterpage}
\usepackage{framed}
\usepackage{nicefrac}

{\endMakeFramed}

\DeclareRobustCommand{\parhead}[1]{\textbf{#1}~}

\usepackage{lineno}

\usepackage{ragged2e}

\usepackage{marginnote}

\newcounter{parcount}

\usepackage{graphicx}
\usepackage[labelfont=bf]{caption}
\usepackage{wrapfig}

\usepackage{booktabs}
\usepackage{multirow}
\usepackage{longtable}

\usepackage[sort]{natbib}
\usepackage{bibunits}

\usepackage{listings}
\usepackage{fancyvrb}
\fvset{fontsize=\normalsize}
\usepackage{algorithmic}

\usepackage[colorlinks,linktoc=all]{hyperref}
\usepackage[all]{hypcap}
\hypersetup{citecolor=violet}
\hypersetup{linkcolor=black}
\hypersetup{urlcolor=violet}

\usepackage[nameinlink]{cleveref}

\usepackage[acronym,smallcaps,nowarn]{glossaries}

\usepackage{listings}
\usepackage{lstbayes}
\lstdefinestyle{mystyle}{
    commentstyle=\color{OliveGreen},
    numberstyle=\tiny\color{black!60},
    stringstyle=\color{BrickRed},
    basicstyle=\ttfamily\scriptsize,
    breakatwhitespace=false,
    breaklines=true,
    captionpos=b,
    keepspaces=true,
    numbers=none,
    numbersep=5pt,
    showspaces=false,
    showstringspaces=false,
    showtabs=false,
    tabsize=2
}
\usepackage{enumitem}

\newacronym{ELBO}{elbo}{evidence lower bound}
\newacronym[shortplural=vae]{VAE}{vae}{variational autoencoder}
\newacronym{GMVAE}{gmvae}{Gaussian mixture VAE}
\newacronym{GMM}{gmm}{Gaussian mixture model}
\newacronym{AU}{au}{active units}
\newacronym{PPCA}{ppca}{probabilistic principal component analysis}
\newacronym{KL}{kl}{Kullback-Leibler}
\newacronym{LDA}{lda}{latent Dirichlet allocation}
\newacronym{SVI}{svi}{stochastic variational inference}
\newacronym{ICNN}{icnn}{input convex neural network}

\newacronym{ML}{ml}{maximum likelihood}
\newacronym{MLE}{mle}{maximum likelihood estimate}
\newacronym{MCMC}{mcmc}{Markov chain Monte Carlo}
\newacronym{HMC}{hmc}{Hamiltonian Monte Carlo}

\newacronym{LBFGS}{l-bfgs}{limited-memory Broyden-Fletcher-Goldfarb-Shanno}
\newacronym{ADVI}{advi}{automatic differentiation variational inference}
\newacronym{NUTS}{nuts}{No-U-Turn sampler}

\newacronym{GLM}{glm}{generalized linear model}

\newacronym{IF}{if}{influence function}

\newacronym{PF}{pf}{Poisson factorization}

\newacronym{MI}{mi}{mutual information}

\newacronym[\glsshortpluralkey={rpm}]
{RPM}{rpm}{reweighted probabilistic model}

\newacronym{NDCG}{ndcg}{normalized discounted cumulative gain}

\newacronym{MAP}{map}{mean average precision}

\newacronym{IDVAE-MT}{idvae-mt}{identifiable VAE via monotone
transport maps}

\newacronym{IDGMVAE-MT}{idgmvae-mt}{identifiable Gaussian mixture VAE via monotone
transport maps}

\newacronym{IDSVAE-MT}{idsvae-mt}{identifiable sequential VAE via
monotone transport maps}

\newacronym{IDVAE}{lidvae}{latent-identifiable VAE}

\newacronym{IDGMVAE}{lidgmvae}{latent-identifiable Gaussian mixture VAE}

\newacronym{IDSVAE}{lidsvae}{latent-identifiable sequential VAE}
\newacronym{IDMVAE}{lidmvae}{latent-identifiable mixture VAE}

\usepackage{tikz}
\usetikzlibrary{bayesnet}
\usepackage{pgfplots}
\pgfplotsset{compat=newest}
\pgfplotsset{plot coordinates/math parser=false}
\usepgfplotslibrary{statistics}
\usetikzlibrary{calc}
\pgfdeclarelayer{edgelayer}
\pgfdeclarelayer{nodelayer}
\pgfsetlayers{edgelayer,nodelayer,main}

\definecolor{hexcolor0xbfbfbf}{rgb}{0.749,0.749,0.749}

\tikzset{>=latex}
\tikzstyle{none}   = [inner sep=0pt]
\tikzstyle{line}   = [ thick, -, shorten <=1pt, shorten >=1pt ]
\tikzstyle{arrow}  = [ thick,  ->, shorten <=1pt, shorten >=1pt ]
\tikzstyle{ardash} = [ thick dotted, ->, shorten <=1pt, shorten >=1pt ]

\tikzstyle{empty}=[circle,opacity=0.0,text opacity=1.0,minimum width=4pt,minimum height=4pt]
\tikzstyle{box}=[rectangle,fill=White,draw=Black]
\tikzstyle{filled}=[circle,fill=hexcolor0xbfbfbf,draw=Black]
\tikzstyle{hollow}=[circle,fill=White,draw=Black]
\tikzstyle{param}=[rectangle,fill=Black,draw=Black,inner sep=0pt,minimum width=4pt,minimum height=4pt]
\tikzstyle{paramhollow}=[rectangle,fill=White,draw=Black,inner sep=0pt,minimum
width=4pt,minimum height=4pt]

\DeclareRobustCommand{\mb}[1]{\ensuremath{\mathbf{\boldsymbol{#1}}}}

\DeclareMathOperator*{\argmax}{arg\,max}

\crefname{lemma}{lemma}{lemmas}
\Crefname{lemma}{Lemma}{Lemmas}
\crefname{thm}{theorem}{theorems}
\Crefname{thm}{Theorem}{Theorems}
\crefname{prop}{proposition}{propositions}
\Crefname{prop}{Proposition}{Propositions}
\crefname{assumption}{assumption}{assumptions}
\Crefname{assumption}{Assumption}{Assumptions}
\crefname{exmp}{example}{examples}
\Crefname{exmp}{Example}{Examples}
\crefname{defn}{definition}{definitions}
\Crefname{defn}{Definition}{Definitions}
\creflabelformat{equation}{#1#2#3}
\crefname{equation}{eq.}{eqs.}
\Crefname{equation}{Eq.}{Eqs.}

\newtheorem{thm}{Theorem} 
\newtheorem{defn}{Definition} 
\newtheorem{prop}[thm]{Proposition}
\newtheorem{exmp}{Example} 

\newcommand{\mbx}{\mb{x}}

\newcommand{\mbz}{\mb{z}}

\newcommand\dif{\mathop{}\!\mathrm{d}}

\newcommand{\E}[2]{\mathbb{E}_{#1}\left[#2\right]}

\newcommand{\cN}{\mathcal{N}}

\newcommand{\g}{\, | \,}
\newcommand{\s}{\, ; \,}

\usepackage{times}
\usepackage{adjustbox}
\usepackage{tcolorbox}
\usepackage{enumitem}
\usepackage{natbib}
\usepackage{booktabs,arydshln}
\usepackage{cuted}
\usepackage{wrapfig}
\usepackage{subcaption}

\allowdisplaybreaks

\everypar=\expandafter{\the\everypar\loosness=-1}
\title{Posterior Collapse and \\ Latent Variable Non-identifiability}

\author{
Yixin Wang\\
University of Michigan\\
\texttt{yixinw@umich.edu}\\
\And
David M. Blei\\
Columbia University\\
\texttt{david.blei@columbia.edu}\\
\And
John P. Cunningham\\
Columbia University\\
\texttt{jpc2181@columbia.edu}\\
}

\begin{document}

\maketitle


\begin{bibunit}[alp]
\begin{abstract}
Variational autoencoders model high-dimensional data by positing
low-dimensional latent variables that are mapped through a flexible
distribution parametrized by a neural network. Unfortunately,
variational autoencoders often suffer from posterior collapse: the
posterior of the latent variables is equal to its prior, rendering the
variational autoencoder useless as a means to produce meaningful
representations. Existing approaches to posterior collapse often
attribute it to the use of neural networks or optimization issues due
to variational approximation. In this paper, we consider posterior
collapse as a problem of latent variable non-identifiability. We prove
that the posterior collapses if and only if the latent variables are
non-identifiable in the generative model. This fact implies that
posterior collapse is not a phenomenon specific to the use of flexible
distributions or approximate inference. Rather, it can occur in
classical probabilistic models even with exact inference, which we
also demonstrate. Based on these results, we propose a class of
latent-identifiable variational autoencoders, deep generative models
which enforce identifiability without sacrificing flexibility. This
model class resolves the problem of latent variable
non-identifiability by leveraging bijective Brenier maps and
parameterizing them with input convex neural networks, without special
variational inference objectives or optimization tricks. Across
synthetic and real datasets, latent-identifiable variational
autoencoders outperform existing methods in mitigating posterior
collapse and providing meaningful representations of the data.
\end{abstract}


\section{Introduction}
\label{sec:introduction}

\glsresetall

\Glspl{VAE} are powerful generative models for high-dimensional
data~\citep{rezende2014stochastic,diederik2014auto}. Their key idea is
to combine the inference principles of probabilistic modeling with the
flexibility of neural networks. In a \gls{VAE}, each datapoint is
independently generated by a low-dimensional latent variable drawn
from a prior, then mapped to a flexible distribution parametrized by a
neural network.

Unfortunately, \gls{VAE} often suffer from posterior collapse, an
important and widely studied phenomenon where the posterior of the
latent variables is equal to
prior~\citep{bowman2016generating,chen2016variational,zhao2018unsupervised,oord2017neural}.
This phenomenon is also known as latent variable collapse, KL
vanishing, and over-pruning. Posterior collapse renders the \gls{VAE}
useless to produce meaningful representations, in so much as its
per-datapoint latent variables all have the exact same posterior.

Posterior collapse is commonly observed in the \gls{VAE} whose
generative model is highly flexible, leading to the common speculation
that posterior collapse occurs because \gls{VAE} involve flexible
neural networks in the generative model~\citep{dai2019usual}, or
because it uses variational
inference~\citep{yacoby2020characterizing}. Based on these hypotheses,
many of the proposed strategies for mitigating posterior collapse thus
focus on modifying the variational inference objective (e.g.
\cite{razavi2019preventing}), designing special optimization schemes
for variational inference in \gls{VAE} (e.g.
\cite{li2019surprisingly,kim2018semi,he2019lagging}), or limiting the
capacity of the generative model (e.g.
\cite{yang2017improved,bowman2016generating,gulrajani2016pixelvae}.)

In this paper, we consider posterior collapse as a problem of latent
variable non-identifiability. We prove that posterior collapse occurs
if and only if the latent variable is non-identifiable in the
generative model, which loosely means the likelihood function does not
depend on the latent
variable~\citep{raue2009structural,wieland2021structural,poirier1998revising}.
Below, we formally establish this equivalence by appealing to recent
results in Bayesian
non-identifiability~\citep{san2010bayesian,raue2013joining,raue2009structural,xie2006measures,poirier1998revising}.

More broadly, the relationship between posterior collapse and latent
variable non-identifiability implies that posterior collapse is not a
phenomenon specific to the use of neural networks or variational
inference. Rather, it can also occur in classical probabilistic models
fitted with exact inference methods, such as Gaussian mixture models
and \gls{PPCA}. This relationship also leads to a new perspective on
existing methods for avoiding posterior collapse, such as the
delta-VAE~\citep{razavi2019preventing} or the $\beta$-VAE
\citep{higgins2016beta}. These methods heuristically adjust the
approximate inference procedure embedded in the optimization of the
model parameters. Though originally motivated by the goal of patching
the variational objective, the results here suggest that these
adjustments are useful because they help avoid parameters at which the
latent variable is non-identifiable and, consequently, avoid posterior
collapse.

The relationship between posterior collapse and non-identifiability
points to a direct solution to the problem: we must make the latent
variable identifiable. To this end, we propose latent-identifiable
\gls{VAE}, a class of \gls{VAE} that is as flexible as classical
\gls{VAE} while also being identifiable. Latent-identifiable \gls{VAE}
resolves the latent variable non-identifiability by leveraging Brenier
maps~\citep{peyre2019computational,mccann1995existence} and
parameterizing them with input-convex neural
networks~\citep{amos2017input,makkuva2019optimal}. Inference on
identifiable \gls{VAE} uses the standard variational inference
objective, without special modifications or optimization tricks.
Across synthetic and real datasets, we show that identifiable
\gls{VAE} mitigates posterior collapse without sacrificing fidelity to
the data.

\parhead{Related work.} Existing approaches to avoiding posterior
collapse often modify the variational inference objective, design new
initialization or optimization schemes for \gls{VAE}, or add neural
network links between each data point and their latent
variables~\citep{bowman2016generating,hoffman2016elbo,sonderby2016train,kingma2016improved,chen2016variational,zhao2018unsupervised,yeung2017tackling,alemi2017fixing,fu2019cyclical,asperti2019variational,li2019surprisingly,seybold2019dueling,zhao2020discretized,havrylov2020preventing,dieng2018avoiding,he2019lagging,kim2018semi,razavi2019preventing,shu2016,tomczak2017vae,oord2017neural,gulrajani2016pixelvae,malloe2019}.
Several recent papers also attempt to provide explanations for
posterior collapse. \citet{chen2016variational} explains how the
inexact variational approximation can lead to inefficiency of coding
in VAE, which could lead to posterior collapse due to a form of
information preference. \citet{dai2019usual} argues that posterior
collapse can be partially attributed to the local optima in training
\gls{VAE} with deep neural networks. \citet{lucas2019don} shows that
posterior collapse is not specific to the variational inference
training objective; absent a variational approximation, the log
marginal likelihood of \gls{PPCA} has bad local optima that can lead
to posterior collapse. \citet{yacoby2020characterizing} discusses how
variational approximation can select an undesirable generative model
when the generative model parameters are non-identifiable. In contrast
to these works, we consider posterior collapse solely as a problem of
latent variable non-identifiability, and not of optimization,
variational approximations, or neural networks per se. We use this
result to propose the identifiable \gls{VAE} as a way to directly
avoid posterior collapse.

Outside \gls{VAE}, latent variable identifiability in probabilistic
models has long been studied in the statistics
literature~\cite{san2010bayesian,raue2013joining,raue2009structural,xie2006measures,poirier1998revising,wieland2021structural,raue2009structural}.
More recently, \citet{identifying2017} studies the effect of latent
variable identifiability on Bayesian computation for Gaussian
mixtures.  \citet{khemakhem2019variational,khemakhem2020icebeem} propose
to resolve the non-identifiability in deep generative models by
appealing to auxiliary data. \citet{kumar2020implicit} study how the
variational family can help resolve the non-identifiability of
\gls{VAE}. These works address the identifiability issue for a
different goal: they develop identifiability conditions for different
subsets of \gls{VAE}, aiming for recovering true causal factors of the
data and improving disentanglement or out-of-distribution
generalization. Related to these papers, we demonstrate posterior
collapse as an additional way that the concept of identifiability,
though classical, can be instrumental in modern probabilistic
modeling. Considering identifiability leads to new solutions to
posterior collapse.

\parhead{Contributions.} We prove that posterior collapse occurs if
and only if the latent variable in the generative model is
non-identifiable. We then propose latent-identifiable \gls{VAE}, a
class of \gls{VAE} that are as flexible as classical \gls{VAE} but
have latent variables that are provably identifiable. Across synthetic
and real datasets, we demonstrate that latent-identifiable \gls{VAE}
mitigates posterior collapse without modifying \gls{VAE} objectives or
applying special optimization tricks.


\glsreset{VAE}
\section{Posterior collapse and latent variable non-identifiability}
\label{sec:methods}

\glsreset{GMM}
\glsreset{PPCA}

Consider a dataset $\mbx = (x_1, \ldots, x_n)$; each datapoint is
$m$-dimensional. Positing $n$ latent variables $\mbz = (z_1, \ldots,
z_n)$, a \gls{VAE} assumes that each datapoint $x_i$ is generated by a
$K$-dimensional latent variable $z_i$:
\begin{align}
z_i &\sim p(z_i),\qquad
x_i \g z_i \sim p(x_i\g z_i\s \theta)=\mathrm{EF}(x_i\g f_\theta(z_i)),\label{eq:vae-gen1}
\end{align}
where $x_i$ follows an exponential family distribution with parameters
$f_\theta(z_i)$; $f_\theta$ parameterizes the conditional likelihood.
In a deep generative model $f_{\theta}$ is a parameterized neural
network. Classical probabilistic models like Gaussian mixture
model~\citep{reynolds2009gaussian} and probabilistic
PCA~\citep{tipping1999probabilistic,collins2001generalization,roweis1998algorithms,roweis1999unifying}
are also special cases of \Cref{eq:vae-gen1}.

To fit the model, \gls{VAE} optimizes the parameters $\theta$ by
maximizing a variational approximation of the log marginal likelihood.
After finding an optimal $\hat{\theta}$, we can form a representation
of the data using the approximate posterior $q_{\hat{\phi}}(z\g x)$
with variational parameters $\hat{\phi}$ or its expectation
$\E{q_{\hat{\phi}}(z\g x)}{z\g x}$.

Note that here we abstract away computational considerations and
consider the ideal case where the variational approximation is exact.
This choice is sensible: if the exact posterior suffers from posterior
collapse then so will the approximate posterior (a variational
approximation cannot ``uncollapse'' a collapsed posterior).  That said
we also note that there exist in practice situations where variational
inference alone can lead to posterior collapse. A notable example is
when the variational approximating family is overly restrictive: it is
then possible to have non-collapsing exact posteriors but collapsing
approximate posteriors.

\subsection{Posterior collapse $\Leftrightarrow$ Latent variable
non-identifiability}

We first define posterior collapse and latent variable
non-identifiability, then proving their connection.

\begin{defn}[Posterior
collapse~\citep{bowman2016generating,chen2016variational,zhao2018unsupervised,oord2017neural}\label{defn:posterior-collapse}]
Given a probability model $p(\mbx, \mbz\s \theta)$, a parameter value
$\theta=\hat{\theta}$, and a dataset $\mbx=(x_1,
\ldots, x_n)$, the posterior of the latent variables $\mbz$ collapses
if
\begin{align}
\label{eq:exact-posterior-collapse}
p(\mbz \g \mbx \s \hat{\theta}) = p(\mbz).
\end{align}
\end{defn}
The posterior collapse phenomenon can occur in a variety of
probabilistic models and with different latent variables. When the
probability model is a \gls{VAE}, it only has local latent variables
$\mbz = (z_1, \ldots, z_n)$, and \Cref{eq:exact-posterior-collapse} is
equivalent to the common definition of posterior collapse $p(z_i\g
x_i\s \hat{\theta}) = p(z_i)$ for all
$i$~\citep{lucas2019don,razavi2019preventing,havrylov2020preventing,dieng2018avoiding}.
Posterior collapse has also been observed in Gaussian mixture
models~\citep{identifying2017}; the posterior of the latent mixture
weights resembles their prior when the number of mixture components in
the model is larger than that of the data generating process.
Regardless of the model, when posterior collapse occurs, it prevents
the latent variable from providing meaningful summary of the dataset.

\begin{defn}[Latent variable
non-identifiability~\citep{raue2009structural,wieland2021structural}]
\label{defn:non-id}
Given a likelihood function $p(\mbx\g\mbz \s \theta)$, a
parameter value $\theta=\hat{\theta}$, and a dataset $\mbx=(x_1,
\ldots, x_n)$, the latent variable $\mbz$ is non-identifiable if
\begin{align}
  p(\mbx \g \mbz=\tilde{\mbz}' \s \hat{\theta}) = p(\mbx \g
  \mbz=\tilde{\mbz} \s \hat{\theta})
\qquad
\forall \tilde{\mbz}', \tilde{\mbz} \in \mathcal{Z},
\label{eq:non-id-def}
\end{align}
where $\mathcal{Z}$ denotes the domain of $\mbz$, and $\tilde{\mbz}',
\tilde{\mbz}$ refer to two arbitrary values the latent variable $\mbz$
can take. As a consequence, for any prior $p(z)$ on $z$, we have the
conditional likelihood equal to the marginal $p(\mbx \g
\mbz=\tilde{\mbz} \s \hat{\theta}) = \int p(\mbx \g \mbz \s
\hat{\theta}) p(\mbz) \dif \mbz = p(\mbx
\s
\hat{\theta})  \quad \forall \tilde{\mbz} \in \mathcal{Z}.$
\end{defn}
\Cref{defn:non-id} says a latent variable $\mbz$ is non-identifiable
when the likelihood of the dataset $\mbx$ does not depend on $\mbz$.
It is also known as practical
non-identifiability~\citep{raue2009structural,wieland2021structural}
and is closely related to the definition of $\mbz$ being conditionally
non-identifiable (or conditionally uninformative) given
$\hat{\theta}$~\citep{san2010bayesian,raue2013joining,raue2009structural,xie2006measures,poirier1998revising}.
To enforce latent variable identifiability, it is sufficient to ensure
that the likelihood $p(\mbx \g \mbz, \theta)$ is an injective (a.k.a.
one-to-one) function of $\mbz$ for all $\theta$. If this condition
holds then
\begin{align}
\tilde{\mbz}'\ne \tilde{\mbz} \qquad
\Rightarrow \qquad
p(\mbx \g \mbz=\tilde{\mbz}' \s \hat{\theta}) \ne p(\mbx \g
\mbz=\tilde{\mbz} \s \hat{\theta}).
\end{align}

Note that latent variable non-identifiability only requires
\Cref{eq:non-id-def} be true for a given dataset $\mbx$ and parameter
value $\hat{\theta}$. Thus a latent variable may be identifiable in a
model given one dataset but not another, and at one $\theta$ but not
another. See examples in \Cref{exmp:PPCA}.

Latent variable
identifiability~(\Cref{defn:non-id})~\citep{raue2009structural,wieland2021structural}
differs from model identifiability~\citep{rao1992identifiability}, a
related notion that has also been cited as a contributing factor to
posterior collapse~\citep{yacoby2020characterizing}. Latent variable
identifiability is a weaker requirement: it only requires the latent
variable $\mbz$ be identifiable at a particular parameter value
$\theta=\hat{\theta}$, while model identifiability requires both
$\mbz$ and $\theta$ be identifiable.

We now establish the equivalence between posterior collapse and latent
variable non-identifiability.

\begin{thm}[Latent variable non-identifiability $\Leftrightarrow$
Posterior collapse]
  \label{thm:collapse-nonid-equiv} Consider a probability model
  $p(\mbx, \mbz\s \theta)$, a dataset $\mbx$, and a parameter value
  $\theta=\hat{\theta}$. The local latent variables $\mbz$ are
  non-identifiable at $\hat{\theta}$ if and only if the posterior of
  the latent variable $\mbz$ collapses, $p(\mbz \g \mbx) = p(\mbz)$.
\end{thm}

\begin{proof} To prove that non-identifiability implies posterior
collapse, note that, by Bayes rule,
\begin{align}
p(\mbz \g \mbx\s \hat{\theta}) \propto
p(\mbz) p(\mbx \g \mbz\s \hat{\theta}) = p(\mbz) p(\mbx
\s
\hat{\theta})\propto p(\mbz),
\label{eq:collapse-nonID-proof}
\end{align} where the middle equality is due to the definition of
latent variable non-identifiability. It implies $p(\mbz \g \mbx\s
\hat{\theta}) = p(\mbz)$ as both are densities. To prove that
posterior collapse implies latent variable non-identifiability, we
again invoke Bayes rule. Posterior collapse implies that $p(\mbz) =
p(\mbz \g \mbx\s
\hat{\theta}) \propto p(\mbz)\cdot p(\mbx
\g \mbz\s \hat{\theta})$, which further implies that $p(\mbx \g \mbz\s
\hat{\theta})$ is constant in $\mbz$. If $p(\mbx \g \mbz\s
\hat{\theta})$ nontrivially depends on $\mbz$, then $p(\mbz)$ must be
different from $p(\mbz)p(\mbx \g \mbz\s \hat{\theta})$ as a function
of $\mbz$.
\end{proof}

The proof of \Cref{thm:collapse-nonid-equiv} is straightforward, but
\Cref{thm:collapse-nonid-equiv} has an important implication. It shows
that the problem of posterior collapse mainly arises from the model
and the data, rather than from inference or optimization. If the
maximum likelihood parameters $\hat{\theta}$ of the \gls{VAE} renders
the latent variable $z$ non-identifiable, then we will observe
posterior collapse. \Cref{thm:collapse-nonid-equiv} also clarifies why
posteriors may change from non-collapsed to collapsed (and back) while
fitting a VAE. When fitting a VAE, Some parameter iterates may lead to
posterior collapse; others may not.

\Cref{thm:collapse-nonid-equiv} points to why existing approaches can
help mitigate posterior collapse. Consider the
$\beta$-\gls{VAE}~\citep{higgins2016beta}, the \gls{VAE} lagging
encoder~\citep{he2019lagging}, and the semi-amortized
\gls{VAE}~\citep{kim2018semi}. Though motivated by other perspectives,
these methods modify the optimization objectives or algorithms of
\gls{VAE} to avoid parameter values $\theta$ at which the latent
variable is non-identifiable. The resulting posterior may not
collapse, though the optimal parameters for these algorithms no longer
approximates the maximum likelihood estimate.

\Cref{thm:collapse-nonid-equiv} can also help us understand posterior
collapse observed in practice, which manifests as the phenomenon that
the posterior is approximately (as opposed to exactly) equal to the
prior, $p(\mbz \g \mbx \s \hat{\theta}) \approx p(\mbz)$. In several
empirical studies of \gls{VAE}~(e.g.
\citep{he2019lagging,dieng2018avoiding,kim2018semi}), we observe that
the \gls{KL} divergence between the prior and posterior is close to
zero but not exactly zero, a property that stems from the likelihood
$p(\mbx\g\mbz)$ being nearly constant in the latents $\mbz$. In these
cases, \Cref{thm:collapse-nonid-equiv} provides the intuition that the
latent variable is nearly non-identifiable , $p(\mbx \g \tilde{\mbz}')
\approx p(\mbx \g \tilde{\mbz}), \forall
\tilde{\mbz}, \tilde{\mbz}'$ and so \Cref{eq:exact-posterior-collapse}
holds approximately.

\subsection{Examples of latent variable non-identifiability and
posterior collapse}

\glsreset{GMM}
\glsreset{PPCA}
\glsreset{GMVAE}

We illustrate \Cref{thm:collapse-nonid-equiv} with three examples.
Here we discuss the example of \gls{GMVAE}.  See \Cref{exmp:PPCA} for
\gls{PPCA} and \gls{GMM}.

The \gls{GMVAE}~\citep{shu2016,dilokthanakul2016deep} is the following
model:
\begin{align*}
p(z_i) = \mathrm{Categorical}(1/K),\quad
p(w_i\g z_i\s \mu, \Sigma) = \cN(\mu_{z_i}, \Sigma_{z_i}),\quad
p(x_i \g w_i\s f, \sigma) = \cN(f(w_i), \sigma^2\cdot I_m),
\end{align*}
where $\mu_k$'s are $d$-dimensional, $\Sigma_k$ are $d\times
d$-dimensional, and the parameters are $\theta = (\mu, \Sigma, f,
\sigma^2)$. Suppose the function $f$ is fully flexible; thus
$f(w_i)$ can capture any distribution of the data. The latent
variable of interest is the categorical $\mbz = (z_1, \ldots, z_n)$.
If its posterior collapses, then $p(z_i=k\g\mbx) = 1/K$ for all
$k=1,\ldots, K$.

\label{exmp:GMVAE}

Consider fitting a \gls{GMVAE} model with $K=2$ to a dataset of 5,000
samples.  This dataset is drawn from a \gls{GMVAE} also with $K=2$
well-separated clusters; there is no model misspecification. A
\gls{GMVAE} is typically fit by optimizing the maximum log marginal
likelihood $\hat{\theta} = \argmax_\theta \log p(\mbx \g \theta)$.
Note there may be multiple values of $\theta$ that achieve the global
optimum of this function.

We focus on two likelihood maximizers. One provides latent variable
identifiability and the posterior of $z_i$ does not collapse. The
other does not provide identifiablity; the posterior collapses.

\begin{enumerate}[leftmargin=*]
\item The first likelihood-maximizing parameter $\hat{\theta}_1$
  is the truth; the distribution of the $K$ fitted clusters correspond
  to the $K$ data-generating clusters. Given this parameter, the
  latent variable $z_i$ is identifiable because the $K$
  data-generating clusters are different; different cluster
  memberships $z_i$ must result in different likelihoods $p(x_i\g
  z_i\s \hat{\theta}_1)$.  The posterior of $z_i$ does not collapse.

\item In the second likelihood-maximizing parameter $\hat{\theta}_2$,
  however, all $K$ fitted clusters share the same distribution, each
  of which is equal to the marginal distribution of the data.
  Specifically, $(\mu_k^*, \Sigma_k^*) = (0, I_d)$ for all $k$, and
  each fitted cluster is a mixture of the $K$ original data generating
  clusters, i.e., the marginal.  At this parameter value, the model is
  still able to fully capture the mixture distribution of the data.
  However, all the $K$ mixture components are the same, and thus the
  latent variable $z_i$ is non-identifiable; different cluster
  membership $z_i$ do not result in different likelihoods $p(x_i\g
  z_i\s \hat{\theta}_2)$, and hence the posterior of $z_i$ collapses.
  \Cref{fig:pinwheel_nonid} illustrates a fit of this
  (non-identifiable) \gls{GMVAE} to the pinwheel
  data~\citep{johnson2016composing}. In \Cref{subsec:idvae}, we
  construct an \gls{IDVAE} that avoids this collapse.

\end{enumerate}

Latent variable identifiability is a function of the both the model
and the true data-generating distribution.  Consider fitting the same
\gls{GMVAE} with $K=2$ but to a different dataset of 5,000 samples,
this one drawn from a \gls{GMVAE} with only one cluster.  (There is
model misspecification.)  One maximizing parameter value
$\hat{\theta}_3$ is where both of the fitted clusters correspond to
the true data generating cluster. While this parameter value resembles
that of the first maximizer $\hat{\theta}_1$ above---both correspond
to the true data generating cluster---this dataset leads to a
different situation for latent variable identifiability.  The two
fitted clusters are the same and so different cluster memberships do
not result in different likelihoods of $p(x_i\g z_i\s
\hat{\theta}_3)$.  The latent variable $z_i$ is not identifiable and
its posterior collapses.

\parhead{Takeaways.} The \gls{GMVAE} example in this section (and the
\gls{PPCA} and \gls{GMM} examples in \Cref{exmp:PPCA}) illustrate
different ways that a latent variable can be non-identifiable in a
model and suffer from posterior collapse. They show that even the true
posterior---without variational inference---can collapse in
non-identifiable models. They also illustrate that whether a latent
variable is identifiable can depend on both the model and the data.
Posterior collapse is an intrinsic problem of the model and the data,
rather than specific to the use of neural networks or variational
inference.

The equivalence between posterior collapse and latent variable
non-identifiability in \Cref{thm:collapse-nonid-equiv} also implies
that, to mitigate posterior collapse, we should try to resolve latent
variable non-identifiability. In the next section, we develop such a
class of latent-identifiable \gls{VAE}.

\section{Latent-identifiable \gls{VAE} via Brenier maps}
\label{subsec:idvae}

We now construct latent-identifiable \gls{VAE}, a class of
\gls{VAE} whose latent variables are guaranteed to be identifiable,
and thus the posteriors cannot collapse.

\subsection{The latent-identifiable \gls{VAE} }

To construct the latent-identifiable \gls{VAE}, we rely on a key observation
that, to guarantee latent variable identifiability, it is sufficient
to make the likelihood function $P(x_i\g z_i\s \theta)$ injective for
all values of $\theta$. If the likelihood is injective, then, for any
$\theta$, each value of $z_i$ will lead to a different distribution
$P(x_i\g z_i\s \theta)$.  In particular, this fact will be true for
any optimized $\hat{\theta}$ and so the latent $z_i$ must be
identifiable, regardless of the data. By
\Cref{thm:collapse-nonid-equiv}, its posterior cannot collapse.

Constructing latent-identifiable \gls{VAE} thus amounts to constructing an
injective likelihood function for \gls{VAE}. The construction is based
on a few building blocks of linear and nonlinear injective functions,
then composed into an injective likelihood $p(x_i\g z_i\s
\theta)$ mapping from $\mathcal{Z}^d$ to $\mathcal{X}^m$, where
$\mathcal{Z}$ and $\mathcal{X}$ indicate the set of values $z_i$ and
$x_i$ can take. For example, if $x_i$ is an m-dimensional binary
vector, then $\mathcal{X} = \{0,1\}^m$; if $z_i$ is a $K$-dimensional
real-valued vector, then $\mathcal{Z} = \mathbb{R}^d$.

\parhead{The building blocks of \gls{IDVAE}: Injective functions.} For
linear mappings from $\mathbb{R}^{d_1}$ to $\mathbb{R}^{d_2}$
$(d_2\geq d_1)$, we consider matrix multiplication by a $d_1\times
d_2$-dimensional matrix~$\beta$. For a $d_1$-dimensional variable $z$,
left multiplication by a matrix $\beta^\top$ is injective when $\beta$
has full column rank~\citep{strang1993introduction}. For example, a
matrix with all ones in the diagonal and all other entries being zero
has full column rank.

For nonlinear injective functions, we focus on Brenier
maps~\citep{ball2004elementary,mccann2011five}. A $d$-dimensional
Brenier map is is the gradient of a convex function from
$\mathbb{R}^d$ to $\mathbb{R}$. That is, a Brenier map satisfies $g =
\nabla T$ for some convex function
$T:\mathbb{R}^d\rightarrow\mathbb{R}$. Brenier maps are also known as
a monotone transport map. They are guaranteed to be
bijective~\citep{ball2004elementary,mccann2011five} because their
derivative is the Hessian of a convex $T$, which must be positive
semidefinite and has a nonnegative
determinant~\citep{ball2004elementary}.

To build a \gls{VAE} with Brenier maps, we require a neural network
parametrization of the Brenier map. As Brenier maps are gradients of
convex functions, we begin with the neural network parametrizaton of
convex functions, namely the
\gls{ICNN}~\citep{amos2017input,makkuva2019optimal}. This
parameterization of convex functions will enable  Brenier maps to be
paramterized as the gradient of \gls{ICNN}.

An $L$-layer \gls{ICNN} is a neural network mapping from
$\mathbb{R}^d$ to $\mathbb{R}$. Given an input $u\in \mathbb{R}^d$,
its $l$th layer is
\begin{align}
\label{eq:icnn}
\mbz_0 = \mathbf{u}, \qquad \mbz_{l+1} = h_l(\mathbf{W}_l \mbz_l+\mathbf{A}_l \mathbf{u}+\mathbf{b}_l), \qquad (l=0,
\ldots, L-1), 
\end{align}
where the last layer $\mbz_L$ must be a scalar, $\{\mathbf{W}_l\}$ are
non-negative weight matrices with $\mathbf{W}_0=\mb{0}$. The functions
$\{h_l:\mathbb{R}\rightarrow\mathbb{R}\}$ are convex and
non-decreasing entry-wise activation functions for layer $l$; they are
applied element-wise to the vector $(\mathbf{W}_l
\mbz_l+\mathbf{A}_l \mathbf{u}+\mathbf{b}_l)$. A common choice of
$h_0:\mathbb{R}\rightarrow\mathbb{R}$ is the square of a leaky RELU,
$h_0(x) = (\max(\alpha \cdot x, x))^2$ with $\alpha=0.2$; the
remaining $h_l$'s are set to be a leaky RELU, $h_l(x)=\max(\alpha
\cdot x, x)$. This neural network is called ``input convex'' because
it is guaranteed to be a convex function.

Input convex neural networks can approximate any convex function on a
compact domain in sup norm (Theorem 1 of \citet{chen2018optimal}.)
Given the neural network parameterization of convex functions, we can
parametrize the Brenier map $g_\theta(\cdot)$ as its gradient with
respect to the input $g_\theta(u)=\partial z_L / \partial u.$
This neural network parameterization of Brenier map is a universal
approxiamtor of all Brenier maps on a compact domain, because input
convex neural networks are universal approximators of convex
functions~\citep{chen2018optimal}.

\glsreset{IDVAE}

\parhead{The \gls{IDVAE}.} We construct injective likelihoods for
\gls{IDVAE} by composing two bijective Brenier maps with an injective
matrix multiplication. As the composition of injective and bijective
mappings must be injective, the resulting composition must be
injective.
Suppose $g_{1,\theta}: \mathbb{R}^{K}\rightarrow \mathbb{R}^{K}$ and
$g_{2,\theta}: \mathbb{R}^{D}\rightarrow \mathbb{R}^{D}$ are two
Brenier maps, and $\beta$ is a $K\times D$-dimensional matrix $(D \geq
K)$ with all the main diagonal entries being one and all other entries
being zero. The matrix $\beta^\top$ has full column rank, so
multiplication by $\beta^\top$ is injective. Thus the composition
$g_{2,\theta}(\beta^\top \,g_{1,\theta}(\cdot))$ must be an injective
function from a low-dimensional space $\mathbb{R}^{K}$ to a
high-dimensional space $\mathbb{R}^{D}$.
\glsreset{IDVAE}
\begin{defn}[\Gls{IDVAE} via Brenier maps]\label{defn:idvae} An
\gls{IDVAE} via Brenier maps generates a $D$-dimensional datapoint
$x_i, \in\{1,
\ldots, n\}$ by:
\begin{align}
\label{eq:id-vae}
z_i \sim p(z_i), \qquad x_i \g z_i \sim \mathrm{EF}(x_i\g g_{2,\theta}(\beta^\top \,g_{1,\theta}(z_i) )),
\end{align}
where $\mathrm{EF}$ stands for exponential family distributions; $z_i$
is a $K$-dimensional latent variable, discrete or continuous. The
parameters of the model are $\theta~=~(g_{1,\theta}, g_{2,\theta})$,
where $g_{1,\theta}:
\mathbb{R}^{K}\rightarrow \mathbb{R}^{K}$ and $g_{2,\theta}:
\mathbb{R}^{D}\rightarrow \mathbb{R}^{D}$ are two continuous Brenier
maps. The matrix $\beta$ is a $K\times D$-dimensional matrix $(D \geq
K)$ with all the main diagonal entries being one and all other entries
being zero.
\end{defn}

Contrasting \gls{IDVAE} (\Cref{eq:id-vae}) with the classical
\gls{VAE} (\Cref{eq:vae-gen1}), the \gls{IDVAE} replaces the function
$f_\theta:\mathcal{Z}^K\rightarrow \mathcal{X}^D$ with the injective
mapping $g_{2,\theta}(\beta^\top \,g_{1,\theta}(\cdot))$, composed by
bijective Brenier maps $g_{1,\theta}, g_{2,\theta}$ and a zero-one
matrix $\beta^\top$ with full column rank. As the likelihood functions
of exponential family are injective, the likelihood function $p(x_i\g
z_i\s \theta) = \mathrm{EF}(g_{2,\theta}(\beta^\top
\,g_{1,\theta}(z_i) ))$ of \gls{IDVAE} must be
injective. Therefore, replacing an arbitrary function
$f_\theta:\mathcal{Z}^K\rightarrow \mathcal{X}^D$ with the injective
mapping $g_{2,\theta}(\beta^\top \,g_{1,\theta}(\cdot))$ plays a
crucial role in enforcing identifiability for latent variable $z_i$
and avoiding posterior collapse in
\gls{IDVAE}. As the latent $z_i$ must be identifiable in \gls{IDVAE},
its posterior does not collapse.

Despite its injective likelihood, \gls{IDVAE} are as flexible as
\gls{VAE}; the use of Brenier maps and \gls{ICNN} does not limit the
capacity of the generative model. Loosely, \gls{IDVAE} can model any
distributions in $\mathbb{R}^D$ because Brenier maps can map any given
non-atomic distribution in $\mathbb{R}^d$ to any other one in
$\mathbb{R}^d$~\citep{mccann2011five}. Moreover, the \gls{ICNN}
parametrization is a universal approximator of Brenier
maps~\citep{amos2017input}. We summarize the key properties of
\gls{IDVAE} in the following proposition.

\begin{prop}
\label{prop:IDVAE}
The latent variable $z_i$ is identifiable in
\gls{IDVAE}, i.e. for all $i\in\{1, \ldots, n\}$, we have
\begin{align}
p(x_i \g z_i = \tilde{z}'\s \theta) = p(x_i \g z_i =
\tilde{z} \s \theta) 
\qquad\Rightarrow \qquad\tilde{z}' =
\tilde{z}, \qquad \forall \tilde{z}', \tilde{z}, \theta.
\end{align}
Moreover, for any \gls{VAE}-generated data distribution, there exists
an \gls{IDVAE} that can generate the same distribution. (The proof is
in \Cref{sec:proof-IDVAE}.)
\end{prop}

\subsection{Inference in \gls{IDVAE}} 

Performing inference in \gls{IDVAE} is identical to the classical
\gls{VAE}, as the two \gls{VAE} differ only in their parameter
constraints. To fit an \gls{IDVAE}, we use the classical amortized
inference algorithm of \gls{VAE}; we maximize the \gls{ELBO} of the
log marginal likelihood~\citep{diederik2014auto}.

In general, \gls{IDVAE} are a drop-in replacement for \gls{VAE}. Both
have the same capacity (\Cref{prop:IDVAE}) and share the same
inference algorithm, but \gls{IDVAE} is identifiable and does not
suffer from posterior collapse. The price we pay for \gls{IDVAE} is
computational: the generative model (i.e. decoder) is parametrized
using the gradient of a neural network; its optimization thus requires
calculating gradients of the gradient of a neural network, which
increases the computational complexity of \gls{VAE} inference and can
sometimes challenge optimization. While fitting classical \gls{VAE}
using stochastic gradient descent has $O(k\cdot p)$ computational
complexity, where $k$ is the number of iterations and $p$ is the
number of parameters, fitting latent-identifiable
\gls{VAE} may require $O(k\cdot p^2)$ computational complexity.

\subsection{Extensions of \gls{IDVAE}}

\begin{figure}[t]
\centering
\begin{subfigure}[b]{0.23\textwidth}
\includegraphics[width=\textwidth]{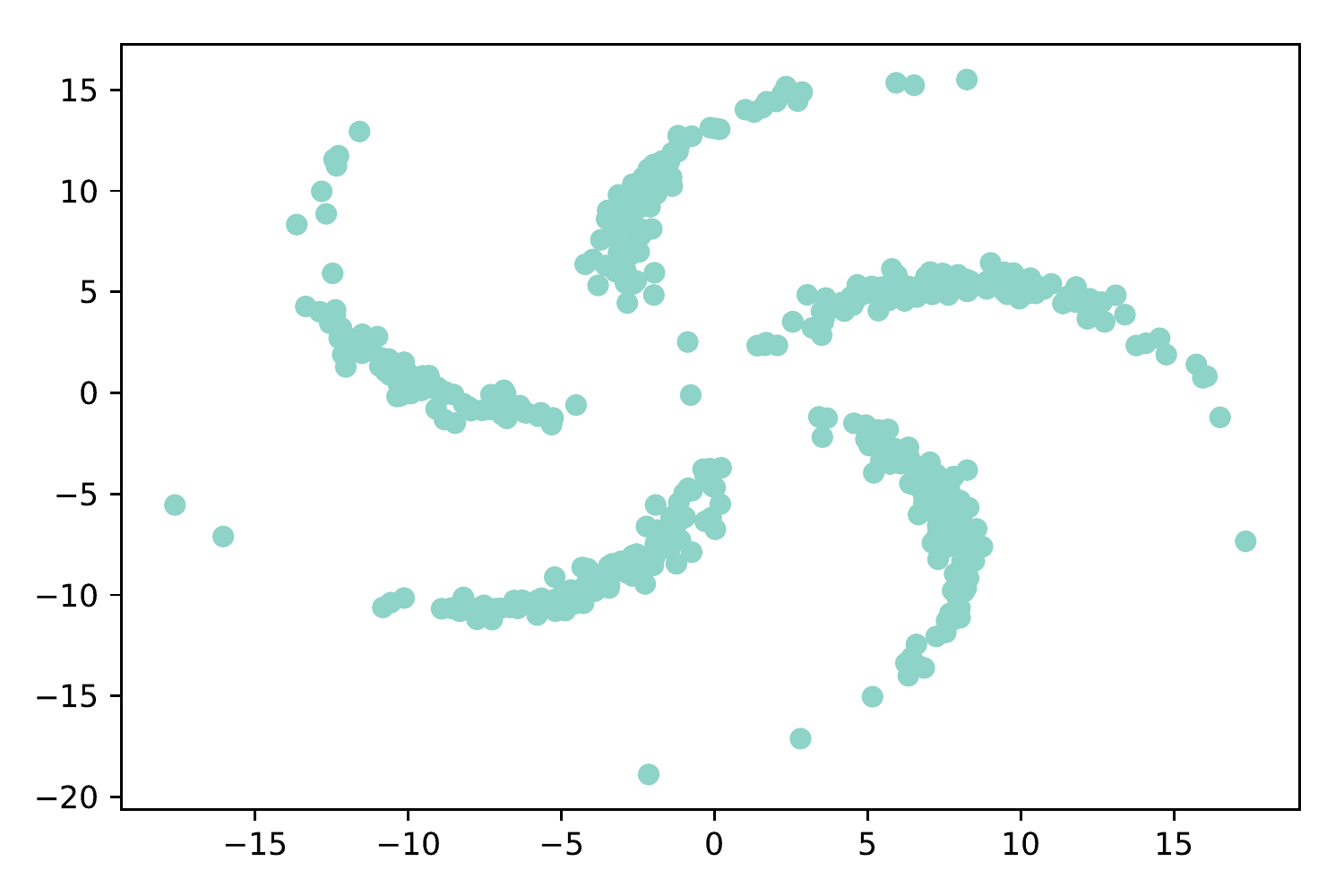}
\caption{Non-ID GMVAE \label{fig:pinwheel_nonid}}
\end{subfigure}%
\begin{subfigure}[b]{0.23\textwidth}
\includegraphics[width=\textwidth]{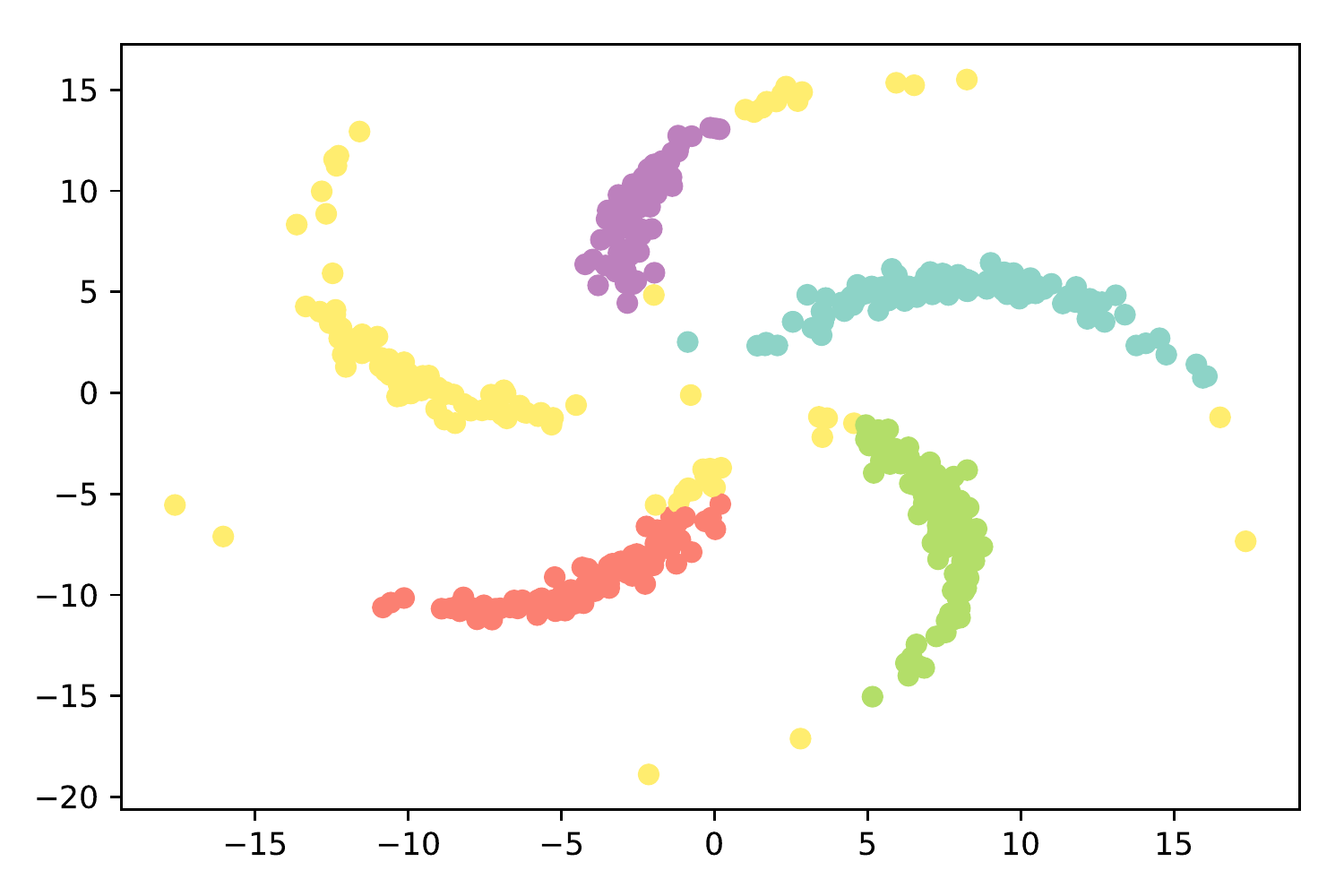}
\caption{IDGMVAE\label{fig:pinwheel_id}}
\end{subfigure}
\begin{subfigure}[b]{0.23\textwidth}
\includegraphics[width=\textwidth]{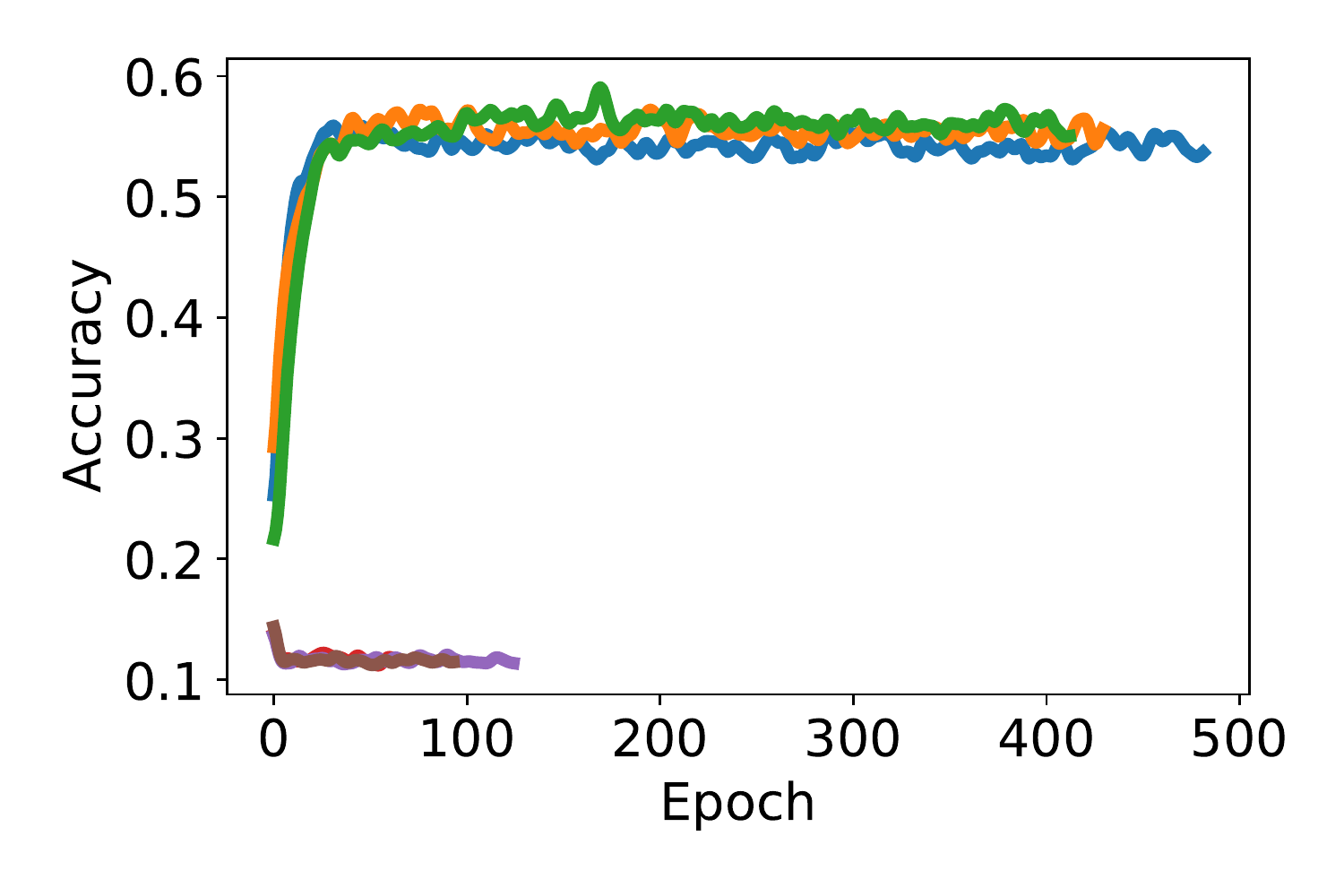}
\caption{Accuracy\label{fig:fashionmist_acc}}
\end{subfigure}
\begin{subfigure}[b]{0.23\textwidth}
\includegraphics[width=\textwidth]{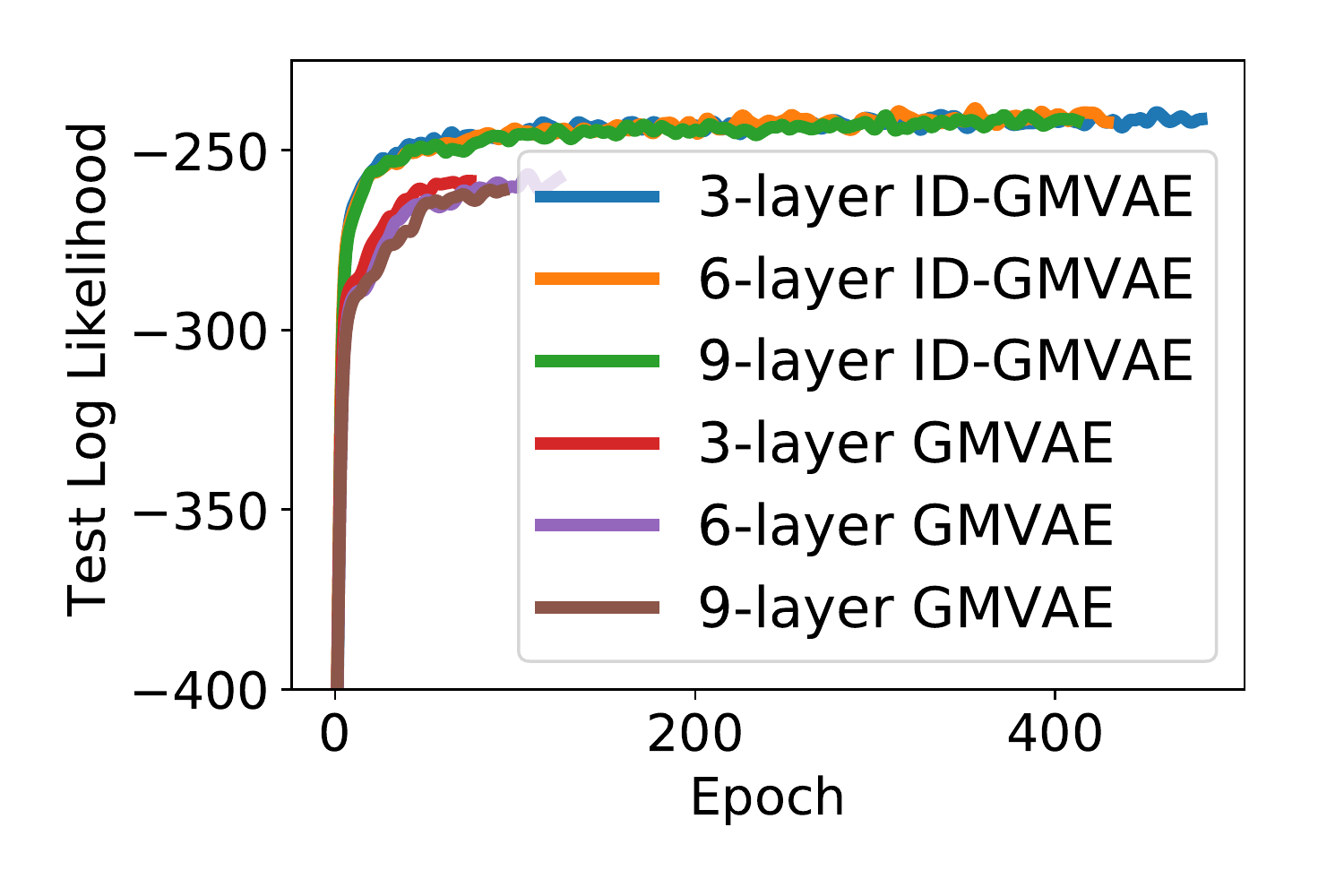}
\caption{Log-likelihood\label{fig:fashionmist_elbo}}
\end{subfigure}
\caption{(a)-(b): The posterior of the
classical
\gls{GMVAE}~\citep{shu2016,kingma2014semi,dilokthanakul2016deep}
collapses when fit to the pinwheel dataset; the latents predict the
same value for all datapoints. The posteriors of
\gls{IDGMVAE}, however, do not collapse and provide meaningful
representations.~\label{fig:gmm_collapse_fig1} (c)-(d) The
latent-identifiable \gls{GMVAE} produces posteriors that are substantially
more informative than
\gls{GMVAE} when fit to fashion MNIST. It also achieves higher test
log likelihood.~\label{fig:fashionmist}}
\end{figure}

The construction of \gls{IDVAE} reveals a general strategy to make the
latent variables of generative models identifiable: replacing
nonlinear mappings with injective nonlinear mappings. We can employ
this strategy to make the latent variables of many other \gls{VAE}
variants identifiable. Below we give two examples, mixture \gls{VAE}
and sequential \gls{VAE}.

The mixture \gls{VAE}, with \gls{GMVAE} as a special case, models the
data with an exponential family mixture and mapped through a flexible
neural network to generate the data. We develop its latent-identifiable
counterpart using Brenier maps.
\begin{exmp}[\Gls{IDMVAE}]\label{defn:id-mixture-vae} An \gls{IDMVAE} generates a
$D$-dimensional datapoint $x_i, i\in\{1,
\ldots, n\}$ by
\begin{align}
z_i \sim \mathrm{Categorical}(1/K), \quad 
w_i \g z_i \sim \mathrm{EF}(w_i\g {\beta_{1}^\top \, z_i }), \quad
x_i \g w_i &\sim \mathrm{EF}(x_i\g g_{2,\theta}(\beta_{2}^\top \,g_{1,\theta}(w_i) )),
\label{eq:idmixturevae-last-layer}
\end{align}
where $W_i$ is a $K$-dimensional one-hot vector that indicates the
cluster assignment. The parameters of the model are
$\theta~=~(g_{1,\theta}, g_{2,\theta})$, where the functions
$g_{1,\theta}:
\mathbb{R}^{M}\rightarrow \mathbb{R}^{M}$ and $g_{2,\theta}:
\mathbb{R}^{D}\rightarrow \mathbb{R}^{D}$ are two continuous Brenier
maps. The matrices $\beta_{1}$ and $\beta_{2}$ are a $K\times
M$-dimensional matrix $(M\geq K)$ and a $M\times D$-dimensional matrix
$(D \geq M)$ respectively, both having all the main diagonal entries
being one and all other entries being zero.
\end{exmp}

The \gls{IDMVAE} differs from the classical mixture
\gls{VAE} in $p(x_i\g z_i)$, where we replace its neural network
mapping with its injective counterpart, i.e. a composition of two
Brenier maps and a matrix multiplication $g_{2,\theta}(\beta_{2}^\top
\,g_{1,\theta}(\cdot))$. As a special case, setting both exponential
families in \Cref{defn:id-mixture-vae} as Gaussian gives us
\gls{IDGMVAE}, which we will use to model images in
\Cref{sec:empirical}.

Next we derive the identifiable counterpart of sequential \gls{VAE},
which models the data with an autoregressive model conditional on the
latents.

\begin{exmp}[\Gls{IDSVAE}]\label{defn:id-svae} An \gls{IDSVAE}
generates a $D$-dimensional datapoint $x_i, i\in\{1,
\ldots, n\}$ by
\begin{align*}
z_i \sim p(z_i), \qquad
x_i \g z_i, x_{<i} \sim \mathrm{EF}(g_{2,\theta}(\beta_{2}^\top \,g_{1,\theta}([z_i, f_\theta(x_{<i})]) )),
\end{align*}
where $x_{<i} = (x_1, \ldots, x_{i-1})$ represents the history of $x$
before the $i$th dimension. The function $f_\theta:
\mathcal{X}_{<i}\rightarrow
\mathbb{R}^{H}$ maps the history $X_{<i}$ into an $H$-dimensional
vector. Finally, $[z_i, f_\theta(x_{<i})]$ is an $(K+H)\times 1$
vector that represents a row-stack of the vectors $(z_i)_{K\times 1}$
and $(f_\theta(x_{<i}))_{H\times 1}$.
\end{exmp}

Similar with mixture \gls{VAE}, the \gls{IDSVAE} also differs from
sequential \gls{VAE} only in its use of $g_{2,\theta}(\beta_{2}^\top
\,g_{1,\theta}(\cdot))$ function in $p(x_i \g z_i, x_{<i})$. We will
use \gls{IDSVAE} to model text in \Cref{sec:empirical}.


\section{Empirical studies}
\label{sec:empirical}


\begin{table}[t]
\footnotesize
  \begin{center}
      \begin{tabular}{lcccccccccccc} 
     \toprule
      &\multicolumn{4}{c}{Fashion-MNIST}&&\multicolumn{4}{c}{Omniglot}\\
       & \textbf{AU}  & \textbf{KL} & \textbf{MI} & \textbf{LL}&& 
       & \textbf{AU}  & \textbf{KL} & \textbf{MI} & \textbf{LL}\\
          \midrule
          \gls{VAE}~\citep{diederik2014auto}& 0.1& 0.2& 0.9& -258.8
          &&&0.02&0.0&0.1&-862.1\\
        SA-\gls{VAE}~\citep{kim2018semi} & 0.2& 0.3& 1.3& -252.2
        &&& 0.1 & 0.2 &1.0  &-853.4  \\
        Lagging \gls{VAE}~\citep{he2019lagging} & 0.4& 0.6& 1.6& -248.5 
        &&& 0.5 &  1.0& 3.6 & -849.4 \\
        $\beta$-\gls{VAE}~\citep{higgins2016beta} ($\beta$=0.2)&0.6&1.2&2.4&-245.3
        &&&0.7&1.4&5.9&-842.6\\        
        \gls{IDGMVAE} (this work) &\bfseries{1.0}&\bfseries{1.6} &\bfseries{2.6}&\bfseries{-242.3}
        &&&\bfseries{1.0}&\bfseries{1.7}&\bfseries{7.5}&\bfseries{-820.3}\\
\bottomrule
    \end{tabular}
    \begin{tabular}{lcccccccccccccccccc} 
     \toprule
      &\multicolumn{4}{c}{Synthetic}&\multicolumn{4}{c}{Yahoo}&\multicolumn{4}{c}{Yelp}\\
       & \textbf{AU}  & \textbf{KL} & \textbf{MI} & \textbf{LL}& \textbf{AU}  & \textbf{KL} & \textbf{MI} & \textbf{LL}& \textbf{AU}  & \textbf{KL} & \textbf{MI} & \textbf{LL} \\
          \midrule
          \gls{VAE}~\citep{diederik2014auto}& 0.0&0.0&0.0&-46.5& 0.0&0.0&0.0&-519.7& 0.0&0.0&0.0&-635.9\\
        SA-\gls{VAE}~\citep{kim2018semi} & 0.4& 0.1& 0.1& -40.2
        &0.2 & 1.0&  0.2& -520.2 &  0.1& 1.9& 0.2&-631.5\\
        Lagging \gls{VAE}~\citep{he2019lagging} &0.5& 0.1& 0.1 & -40.0
        &0.3 & 1.6&  0.4& \bfseries{-518.6} &  0.2& 3.6& 0.1& \bfseries{-631.0}\\          
        $\beta$-\gls{VAE}~\citep{higgins2016beta}
        ($\beta$=0.2)&\bfseries{1.0}&0.1&0.1&
        \bfseries{-39.9}
        &0.5&4.7&0.9&-524.4
        &0.3&\bfseries{10.0}&0.1&-637.3\\
        \gls{IDSVAE} &\bfseries{1.0}&\bfseries{0.5}& \bfseries{0.6}&-40.3
        &\bfseries{0.8}&\bfseries{7.2}& \bfseries{1.1}&-519.5
        &\bfseries{0.7}&9.1&\bfseries{0.9}&-634.2
        \\
\bottomrule
    \end{tabular}
    \vspace{0.2em}
    \caption{Across image and text datasets,
    \gls{IDVAE} outperforms existing \gls{VAE} variants in
    preventing posterior collapse while achieving similar
    goodness-of-fit to the data.
    \label{tab:idvae-empirical}}
  \end{center}
\end{table}

We study \gls{IDVAE} on images and text datasets, finding that
\gls{IDVAE} do not suffer from posterior collapse as we increase the
capacity of the generative model, while achieving similar fits to the
data. We further study
\gls{PPCA}, showing how likelihood functions nearly constant in latent
variables lead to collapsing posterior even with
\gls{MCMC}.

\subsection{\gls{IDVAE} on images and text}

We consider three metrics for evaluating posterior collapse: (1)
\gls{KL} divergence between the posterior and the prior,
$\gls{KL}(q(\mbz\g \mbx)|| p(\mbz))$; (2) Percentange of
\gls{AU}:$\gls{AU} =
\sum_{d=1}^D\mathbb{1}\{\mathrm{Cov}_{p(\mbx)}(\E{q(\mbz\g\mbx)}{\mbz_d})\geq
\epsilon\},$ where $\mbz_d=(z_{1d}, \ldots, z_{nd})$ is the $d$th
dimension of the latent variable $\mbz$ for all the $n$ data points.
In calculating \gls{AU}, we follow \citet{burda2015importance} to
calculate the posterior mean, $(\E{}{z_{1d} \g \mbx_1}, \ldots,
\E{}{z_{nd} \g \mbx_n])}$ for all data points, and calculate the
sample variance of $\E{}{z_{id} \g \mbx_i}$ across $i$'s from this
vector. The threshold $\epsilon$ is chosen to be
0.01~\citep{burda2015importance}; the theoretical maximum of $\%
\gls{AU}$ is one; (3) Approximate \Gls{MI} between $\mbx_i$ and
$\mbz_i$, $I(\mbx, \mbz) =
\E{\mbx}{\E{q(\mbz\g \mbx)}{\log(q(\mbz\g \mbx))}} -
\E{\mbx}{\E{q(\mbz\g \mbx)}{\log(q(\mbz))}}$. We also evaluate the
model fit using the importance weighted estimate of log-likelihood on
a held-out test set~\citep{burda2015importance}. For mixture
\gls{VAE}, we also evaluate the predictive accuracy of the categorical
latents against ground truth labels to quantify their informativeness.

\parhead{Competing methods.} We compare \gls{IDVAE} with the classical
\gls{VAE}~\citep{diederik2014auto}, the $\beta$-\gls{VAE}
($\beta$=0.2)~\citep{higgins2016beta}, the semi-amortized
\gls{VAE}~\citep{kim2018semi}, and the lagging
\gls{VAE}~\citep{he2019lagging}. Throughout the empirical studies, we
use flexible variational approximating families
(RealNVPs~\citep{dinh2016density} for image and
LSTMs~\citep{hochreiter1997long} for text).

\parhead{Results: Images.} We first study \gls{IDGMVAE} on four
subsampled image datasets drawn from
pinwheel~\citep{johnson2016composing}, MNIST~\citep{lecun2010mnist},
Fashion MNIST~\citep{xiao2018fashionmnist}, and
Omniglot~\citep{lake2015human}.
\Cref{fig:pinwheel_nonid,fig:pinwheel_id} illustrate a fit of the
\gls{GMVAE} and the \gls{IDGMVAE} to the pinwheel
data~\citep{johnson2016composing}. The posterior of the \gls{GMVAE}
latents collapse, attributing all datapoints to the same latent
cluster. In contrast, \gls{IDGMVAE} produces categorical latents
faithful to the clustering structure. \Cref{fig:fashionmist} examines
the \gls{IDGMVAE} as we increase the flexibility of the generative
model. \Cref{fig:fashionmist_acc} shows that the categorical latents
of the \gls{IDGMVAE} are substantially more predictive of the true
labels than their classical counterparts. Moreover, its performance
does not degrade as the generative model becomes more flexible.
\Cref{fig:fashionmist_elbo} shows that the \gls{IDGMVAE} consistently
achieve higher test log-likelihood. \Cref{tab:idvae-empirical}
compares different variants of
\gls{VAE} in a 9-layer generative model. Across four datasets,
\gls{IDGMVAE} mitigates posterior collapse. It achieves higher
\gls{AU}, \gls{KL} and \gls{MI} than other variants of \gls{VAE}. It
also achieves a higher test log-likelihood.

\parhead{Results: Text.} We apply \gls{IDSVAE} to three subsampled
text datasets drawn from a synthetic text dataset, the Yahoo dataset,
and the Yelp dataset~\citep{yang2017improved}. The synthetic dataset
is generated from a classical two-layer sequential \gls{VAE} with a
five-dimensional latent. \Cref{tab:idvae-empirical} compares the
\gls{IDSVAE} with the sequential \gls{VAE}. Across the three text
datasets, the \gls{IDSVAE} outperforms other variants of \gls{VAE} in
mitigating posterior collapse, generally achieving a higher \gls{AU},
\gls{KL}, and \gls{MI}.

\subsection{Latent variable non-identifiability and posterior collapse
in PPCA}

\begin{figure}
\centering
\begin{subfigure}{0.25\textwidth}
\centering
\includegraphics[width=\textwidth]{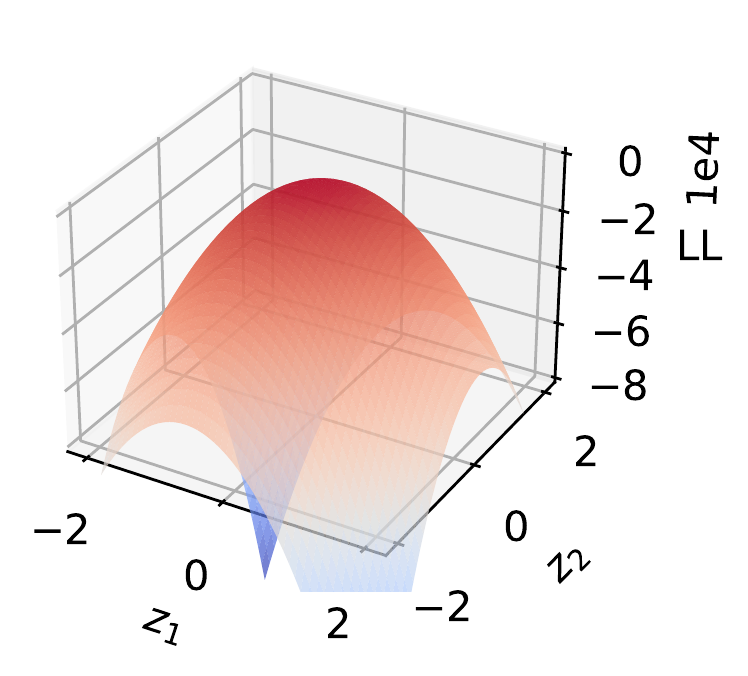}
\includegraphics[width=0.7\textwidth]{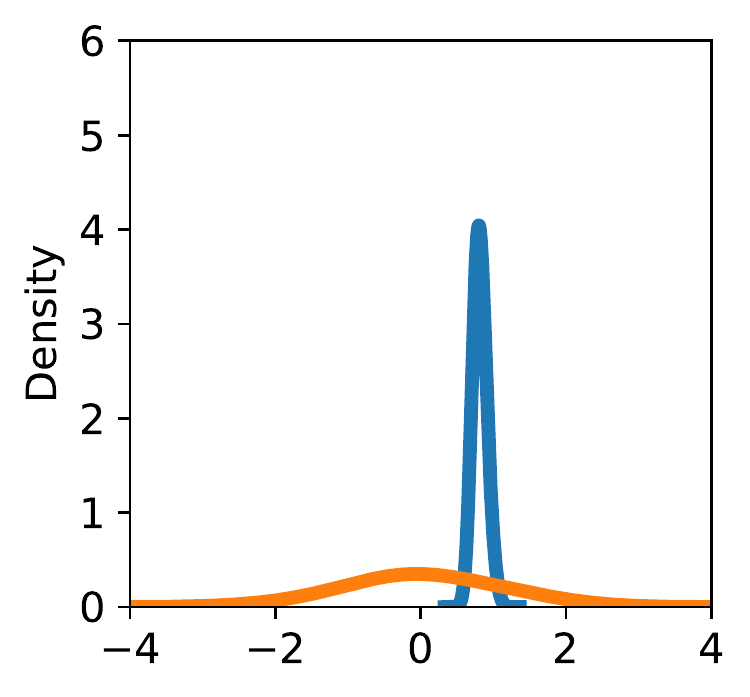}
\caption{$\sigma=0.2$ \label{fig:ppca-varsigma_ll_log2}}
\end{subfigure}%
\begin{subfigure}{0.25\textwidth}
\centering
\includegraphics[width=\textwidth]{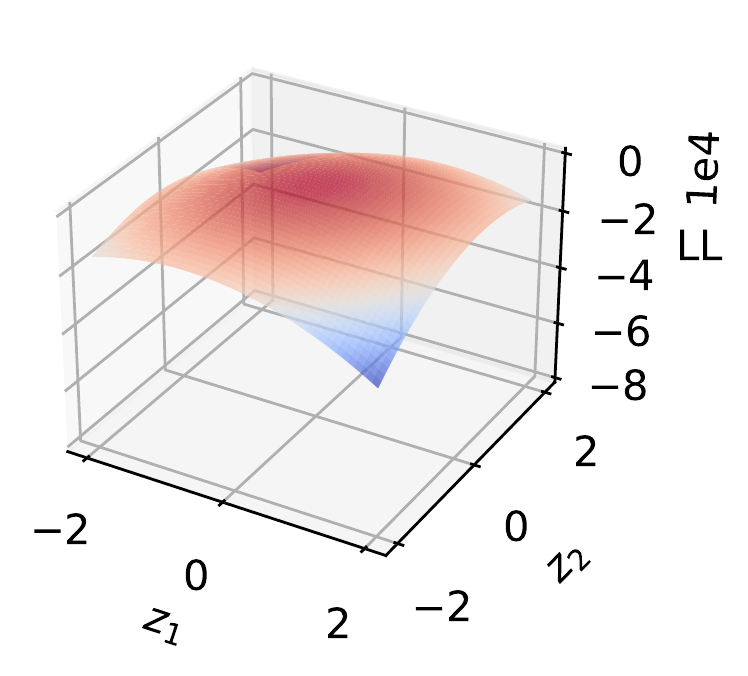}
\includegraphics[width=0.7\textwidth]{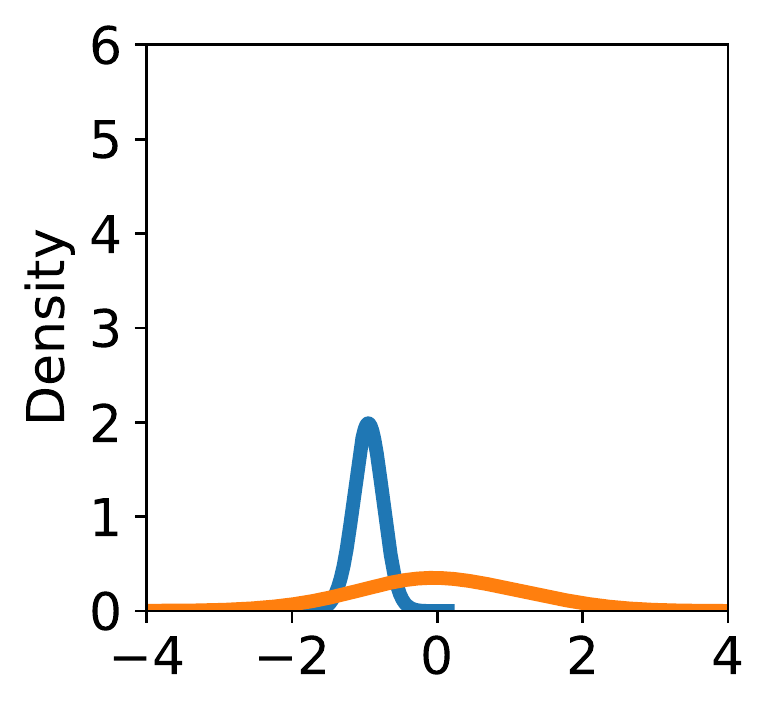}
\caption{$\sigma=0.5$ \label{fig:ppca-varsigma_ll_log5}}
\end{subfigure}%
\begin{subfigure}{0.25\textwidth}
\centering
\includegraphics[width=\textwidth]{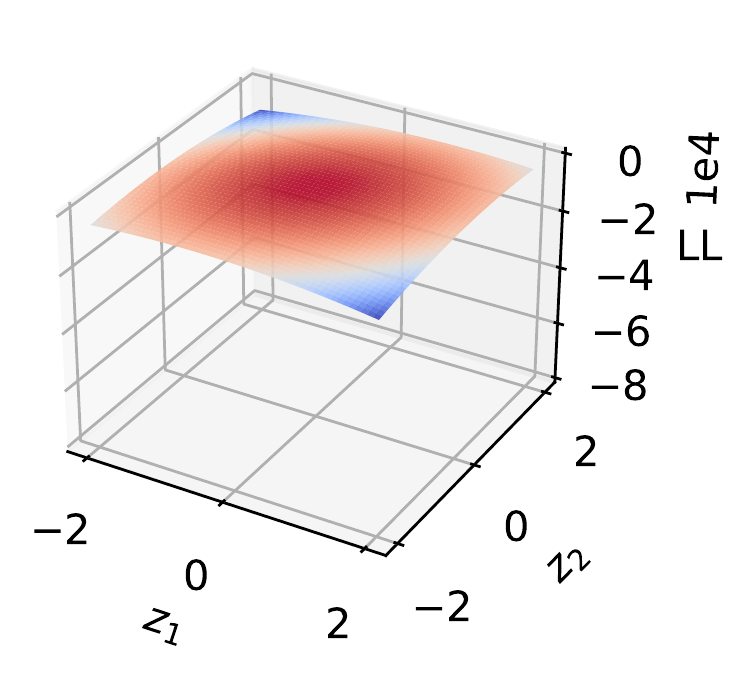}
\includegraphics[width=0.7\textwidth]{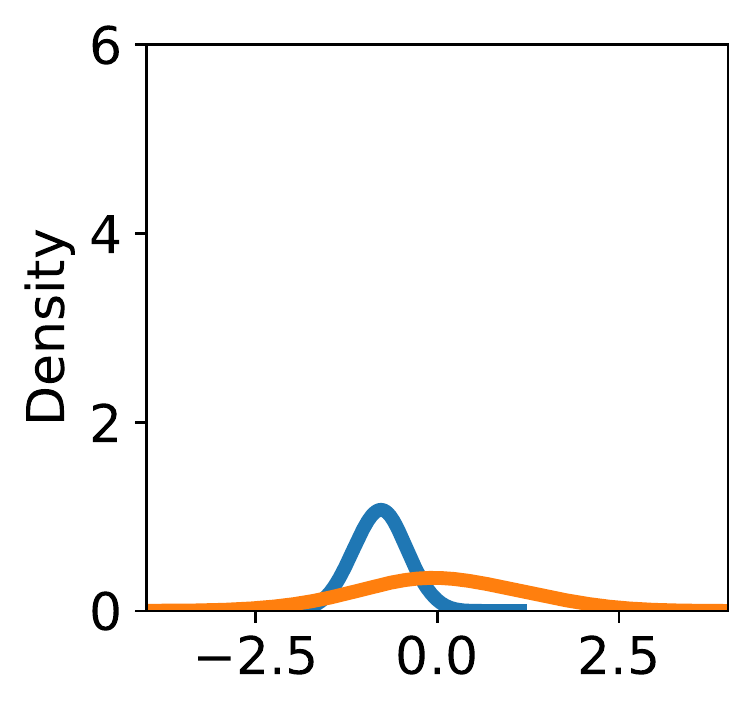}
\caption{$\sigma=1.0$ \label{fig:ppca-varsigma_ll_log10}}
\end{subfigure}%
\begin{subfigure}{0.25\textwidth}
\centering
\includegraphics[width=\textwidth]{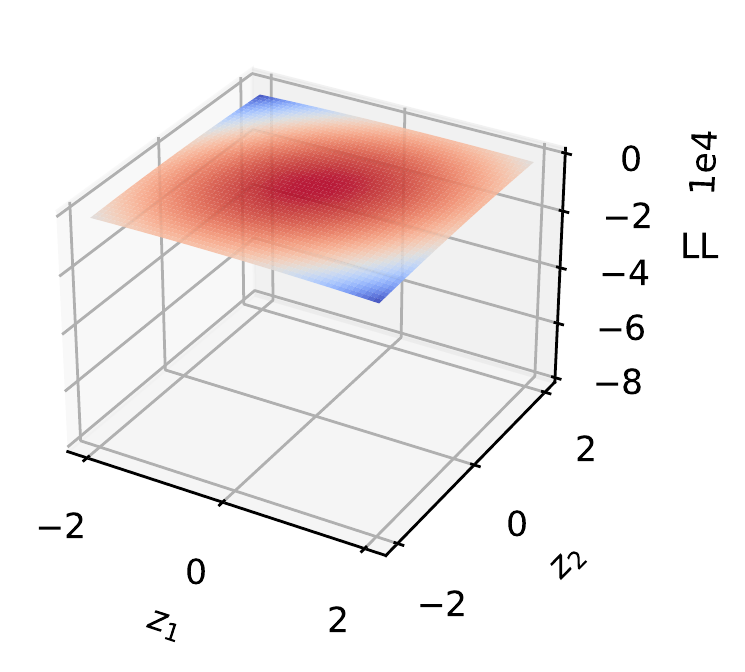}
\includegraphics[width=0.7\textwidth]{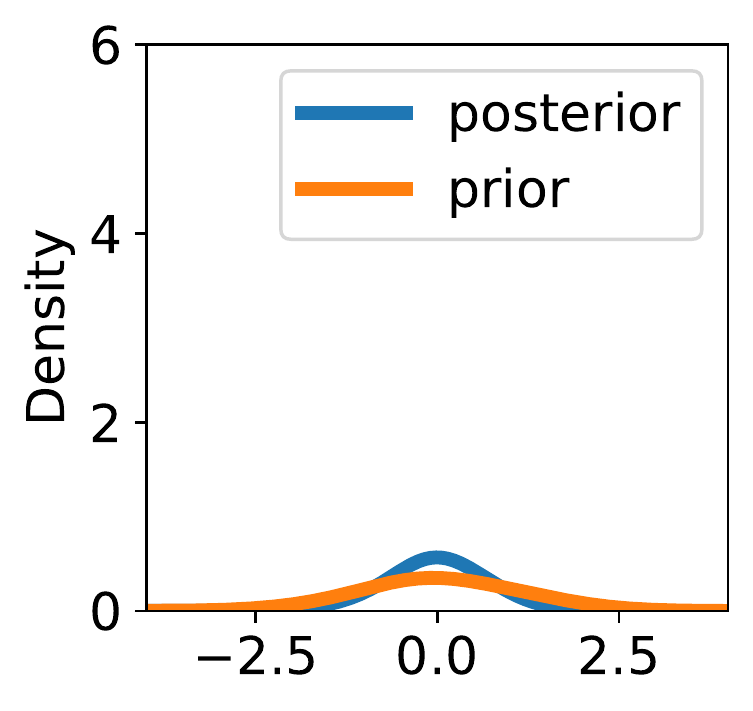}
\caption{$\sigma=1.5$ \label{fig:ppca-varsigma_ll_log10}}
\end{subfigure}%
\caption{As the noise level increases in \gls{PPCA}, the latent
variable becomes closer to non-identifiable because the likelihood and
more susceptible to posterior collapse. Its likelihood surface becomes
flatter and its posterior becomes closer to the prior. Top panel:
Likelihood surface of \gls{PPCA} as a function of the two latents
$z_1, z_2$. When $\sigma$ increase, the likelihood surface becomes
flatter and the latent variables $z_1, z_2$ are closer to
non-identifiable. Bottom panel: Posterior of $z_1$ under different
$\sigma$ values. When $\sigma$ increase, the posterior becomes closer
to the prior. \label{fig:ppca-varsigma}}
\end{figure}

Here we show that the PPCA posterior becomes close to the prior when
the latent variable becomes close to be non-identifiable. We perform
inference using
\gls{HMC}, avoiding the effect of variational approximation on
posterior collapse.

Consider a \gls{PPCA} with two latent dimensions, $p(z_i) = \cN(z_i\s
0, I_2), \,  p(x_i\g z_i \s \theta) = \cN(x_i\s z_i^\top w,
\sigma^2\cdot I_5),$ where the value of $\sigma^2$ is known, $z_i$'s
are the latent variables of interest, and $w$ is the only parameter of
interest. When the noise $\sigma^2$ is set to a large value, the
latent variable $z_i$ may become nearly non-identifiable. The reason
is that the likelihood function $p(x_i\g z_i)$ becomes slower-varying
as $\sigma^2$ increases. For example, \Cref{fig:ppca-varsigma} shows
that the likelihood surface becomes flatter as $\sigma^2$ increases.
Accordingly, the posterior becomes closer to the prior as $\sigma^2$
increases. When $\sigma=1.5$, the posterior collapses. This
non-identifiability argument provides an explanation to the closely
related phenomenon described in Section 6.2 of \citep{lucas2019don}.



\section{Discussion}
\label{sec:discussion}
In this work, we show that the posterior collapse phenomenon is a
problem of latent variable non-identifiability.  It is not specific to
the use of neural networks or particular inference algorithms in
\gls{VAE}. Rather, it is an intrinsic issue of the model and the
dataset. To this end, we propose a class of \gls{IDVAE} via Brenier
maps to resolve latent variable non-identifiability and mitigate
posterior collapse. Across empirical studies, we find that \gls{IDVAE}
outperforms existing methods in mitigating posterior collapse. 

The latent variables of \gls{IDVAE} are guaranteed to be identifiable.
However, it does not guarantee that the latent variables and the
parameters of \gls{IDVAE} are jointly identifiable. In other words,
the \gls{IDVAE} model may not be identifiable even though its latents
are identifiable. This difference between latent variable
identifiability and model identifiability may appear minor. But the
tractability of resolving latent variable identifiability plays a key
role in making non-identifiability a fruitful one perspective of
posterior collapse. To enforce latent variable identifiability, it is
sufficient to ensure that the likelihood $p(\mbx\g \mbz,
\hat{\theta})$ is an injective function of $\mbz$. In contrast,
resolving model identifiability for the general class of
\gls{VAE} remains a long standing open problem, with some recent
progress relying on auxiliary
variables~\citep{khemakhem2019variational,khemakhem2020icebeem}. The
tractability of resolving latent variable identifiability is a key
catalyst of a principled solution to mitigating posterior collapse.

There are a few limitations of this work. One is that the theoretical
argument focuses on the collapse of the exact posterior. The rationale
is that, if the exact posterior collapses, then its variational
approximation must also collapse because variational approximation of
posteriors cannot ``uncollapse'' a posterior. That said, variational
approximation may ``collapse'' a posterior, i.e. the exact posterior
does not collapse but the variational approximate posterior collapses.
The theoretical argument and algorithmic approaches developed in this
work does not apply to this setting, which remains an interesting
venue of future work.

A second limitation is that the latent-identifiable \gls{VAE}
developed in this work bear a higher computational cost than classical
\gls{VAE}. While the latent-identifiable \gls{VAE} ensures the
identifiability of its latent variables and mitigates posterior
collapse, it does come with a price in computation because its
generative model (i.e. decoder) is parametrized using gradients of a
neural network. Fitting the latent-identifiable \gls{VAE} thus
requires calculating gradients of gradients of a neural network,
leading to much higher computational complexity than fitting the
classifical \gls{VAE}. Developing computationally efficient variants
of the latent-identifiable \gls{VAE} is another interesting direction
for future work.

\parhead{Acknowledgments. } We thank Taiga Abe and Gemma Moran for
helpful discussions, and anonymous reviewers for constructive feedback
that improved the manuscript. David Blei is supported by ONR
N00014-17-1-2131, ONR N00014-15-1-2209, NSF CCF-1740833, DARPA SD2
FA8750-18-C-0130, Amazon, and the Simons Foundation. John Cunningham
is supported by the Simons Foundation, McKnight Foundation, Zuckerman
Institute, Grossman Center, and Gatsby Charitable Trust.

\clearpage
{\small\putbib[BIB1]}
\end{bibunit}
\clearpage

\clearpage
\setcounter{page}{1}
\begin{bibunit}[alp]

\onecolumn
{\Large\textbf{Supplementary Materials\\ \\Posterior Collapse and Latent
Variable Non-identifiability}}

\appendix

\section{Examples of posterior collapse continued}

We present two additional examples of posterior collapse,
probabilistic principal component analysis and Gaussian mixture model.

\label{exmp:PPCA}

\glsreset{PPCA}
\subsection{Probabilistic principal component analysis} 

We consider classical \gls{PPCA} and show that its local latent
variables can suffer from posterior collapse at maximum likelihood
parameter values (i.e. global maxima of log marginal likelihood). This
example refines the perspective of \citet{lucas2019don}, which
demonstrated that posterior collapse can occur in \gls{PPCA} absent
any variational approximation but due to local maxima in the log
marginal likelihood. Here we show that posterior collapse can occur
even with global maxima, absent optimization issues due to local
maxima.

Consider a \gls{PPCA} with two latent dimensions,
\begin{align*}
  p(z_i) &= \cN(z_i\g 0, I_2), \\
  p(x_i\g z_i \s \theta) &= \cN(x_i\g z_i^\top w, \sigma^2\cdot I_5),
\end{align*}
where $z_i$'s are the latent variables of interest and others
$\theta=(w, \sigma^2)$ are parameters of the model.

Consider fitting this model to two datasets, each with 500 samples,
focusing on maximum likelihood parameter values.  Depending on the
true distribution of the dataset, \gls{PPCA} may or may not suffer
from posterior collapse.

\begin{figure}[b]
\centering
\begin{subfigure}{0.345\textwidth}
\includegraphics[width=\textwidth]{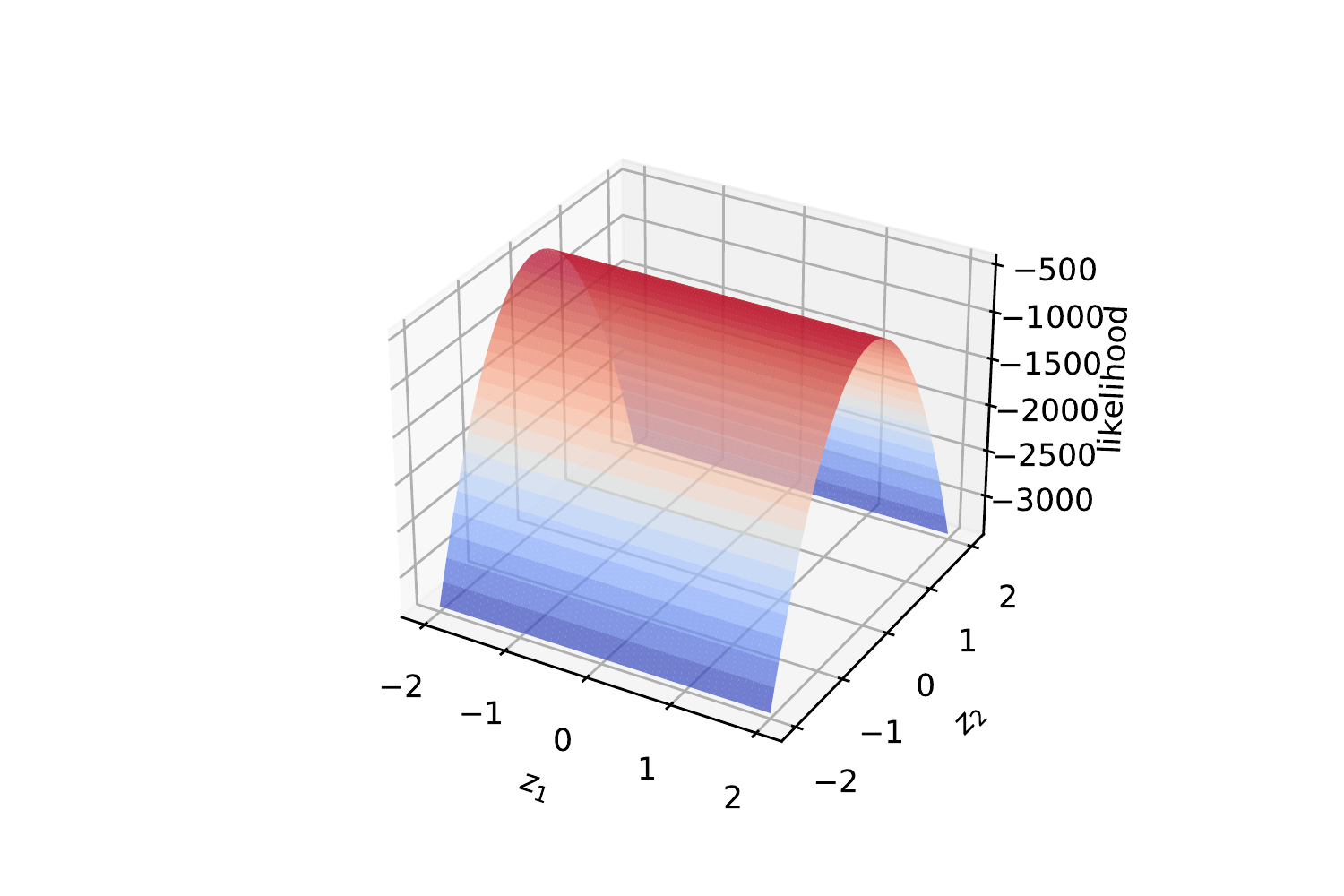}
\caption{Likelihood (1D PPCA) \label{fig:oned_ppca_ll}}
\end{subfigure}%
\begin{subfigure}{0.33\textwidth}
\includegraphics[width=\textwidth]{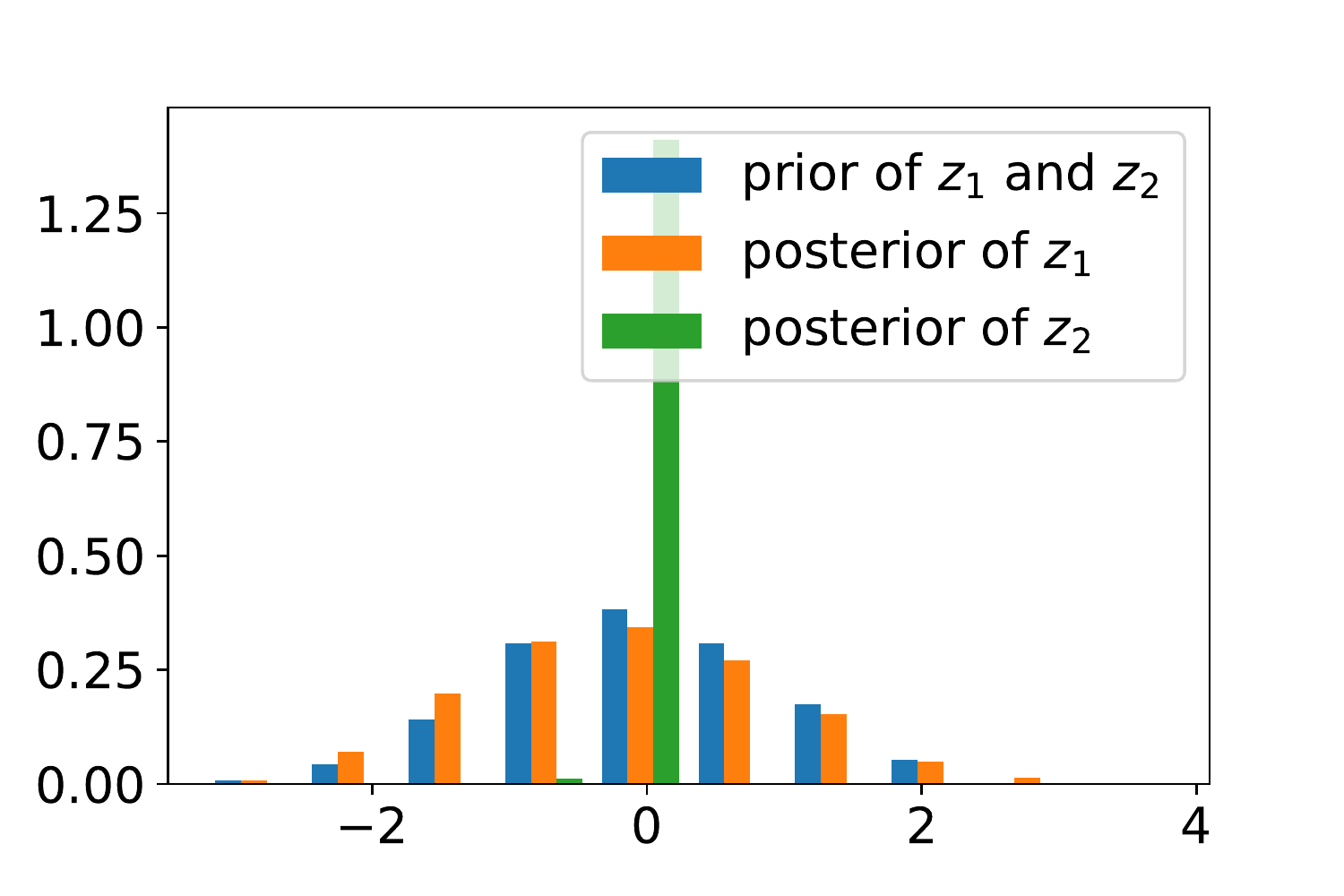}
\caption{Posterior (1D PPCA) \label{fig:oned_ppca_posterior}}
\end{subfigure}%

\begin{subfigure}{0.345\textwidth}
\includegraphics[width=\textwidth]{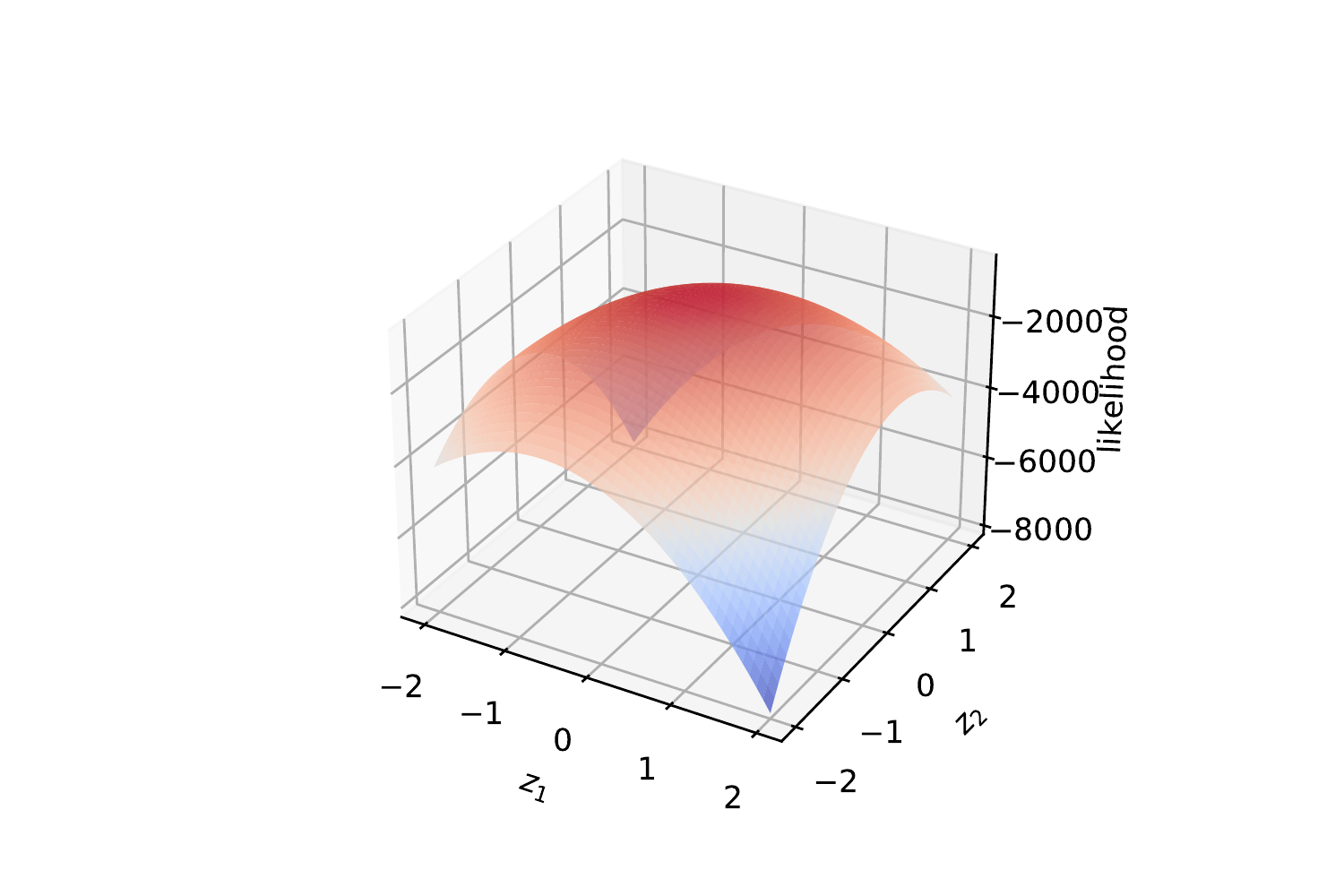}
\caption{Likelihood (2D PPCA) \label{fig:twod_ppca_ll}}
\end{subfigure}
\begin{subfigure}{0.33\textwidth}
\includegraphics[width=\textwidth]{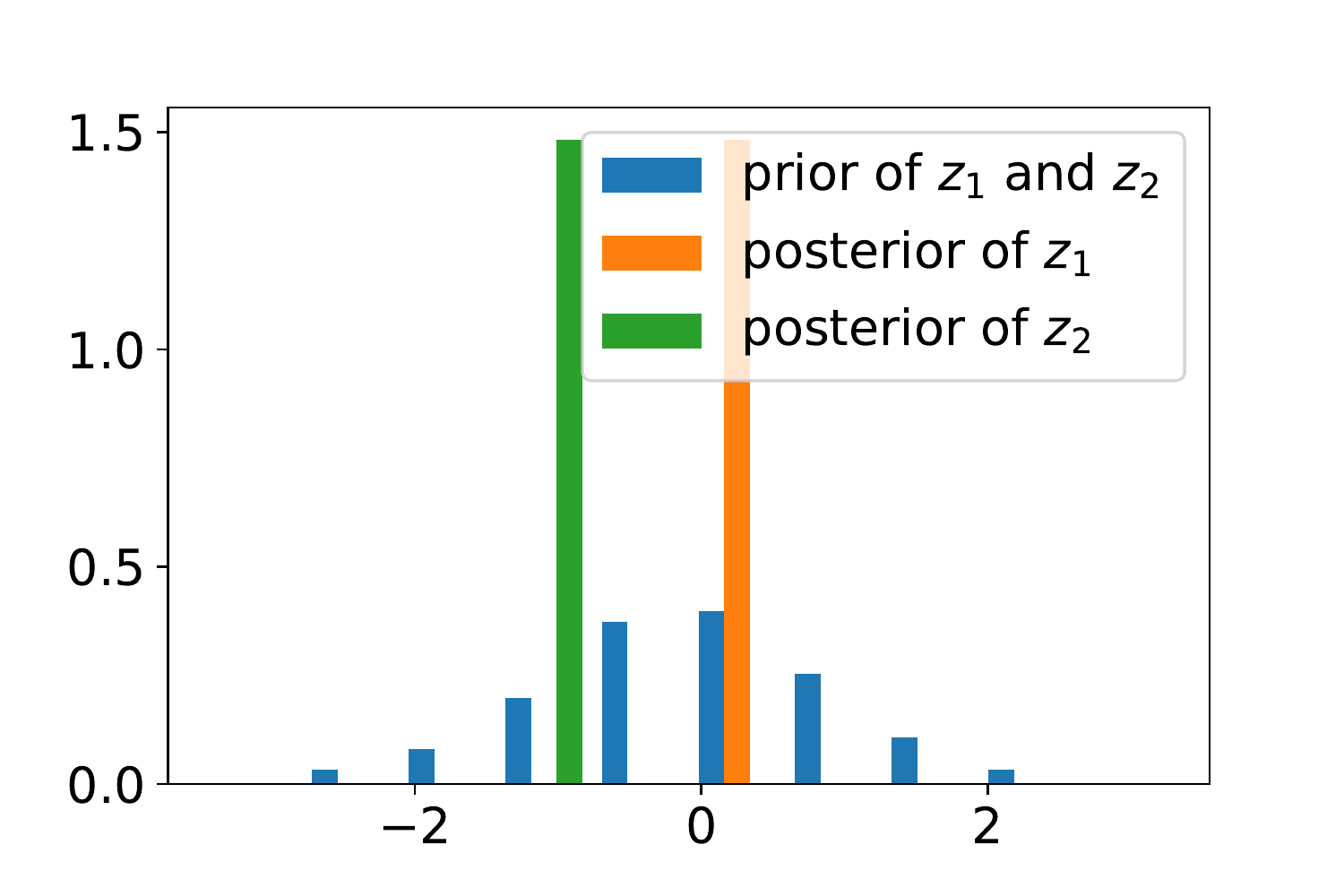}
\caption{Posterior (2D PPCA) \label{fig:twod_ppca_posterior}}
\end{subfigure}
\caption{Fitting \gls{PPCA} with more latent dimensions than enough
leads to non-identifiable local latent variables and collapsed
posteriors. (a)-(b) Fit a two-dimensional \gls{PPCA} to data drawn
from a one-dimensional \gls{PPCA}. The likelihood surface is constant
in one dimension of the latent variable, i.e. this latent variable is
non-identifiable. Hence its corresponding posterior collapses. (c)-(d)
Fit a two-dimensional
\gls{PPCA} to data from a two-dimensional \gls{PPCA} does not suffer
from posterior collapse; its likelihood surface varies in all
dimensions.~\label{fig:ppca_collapse}}
\end{figure}

\begin{enumerate}[leftmargin=*]
\item Sample the data from a one-dimensional \gls{PPCA},
  \begin{align}
    x_i \sim \cN(x_i\g \cN(0, I_1)\cdot\bar{w}_1, \bar{\sigma}_1\cdot
    I_5).
  \end{align} 
  (The model remains two dimensional.) The latent variables $z_i$'s
  are not (fully) identifiable in this case. The reason is that one
  set of maximum likelihood parameters is $\hat{\theta} = (\hat{w},
  \hat{\sigma}) = ([\mb{0}, \bar{w}_1],
  \bar{\sigma}_1)$, i.e. setting one latent dimension as zero and the
  other equal to the true data generating direction. Under this
  $\hat{\theta}$, the likelihood function is constant in the first
  dimension of the latent variable, i.e.  $z_{i1}$; see
  \Cref{fig:oned_ppca_ll}. The posterior of $z_{i1}$ thus collapses,
  matching the prior, while the posterior of $z_{i2}$ stays peaked
  (\Cref{fig:oned_ppca_posterior}).

\item Sample the data from from a two-dimensional \gls{PPCA},
  \begin{align}
    x_i\sim\cN(x_i\g \cN(0, I_2)\cdot\bar{w}_2, \bar{\sigma}_2\cdot
    I_5).
  \end{align}
  The latent variables $z_i$ are identifiable. The
  likelihood function varies against both $z_{i1}$ and $z_{i2}$; the
  posteriors of both $z_{i1}$ and $z_{i2}$ are
  peaked~(\Cref{fig:twod_ppca_ll,fig:twod_ppca_posterior}).
\end{enumerate}

\glsreset{GMM}
\subsection{Gaussian mixture model}
\label{exmp:GMM}

\begin{figure}
\centering
\begin{subfigure}{0.347\textwidth}
\includegraphics[width=\textwidth]{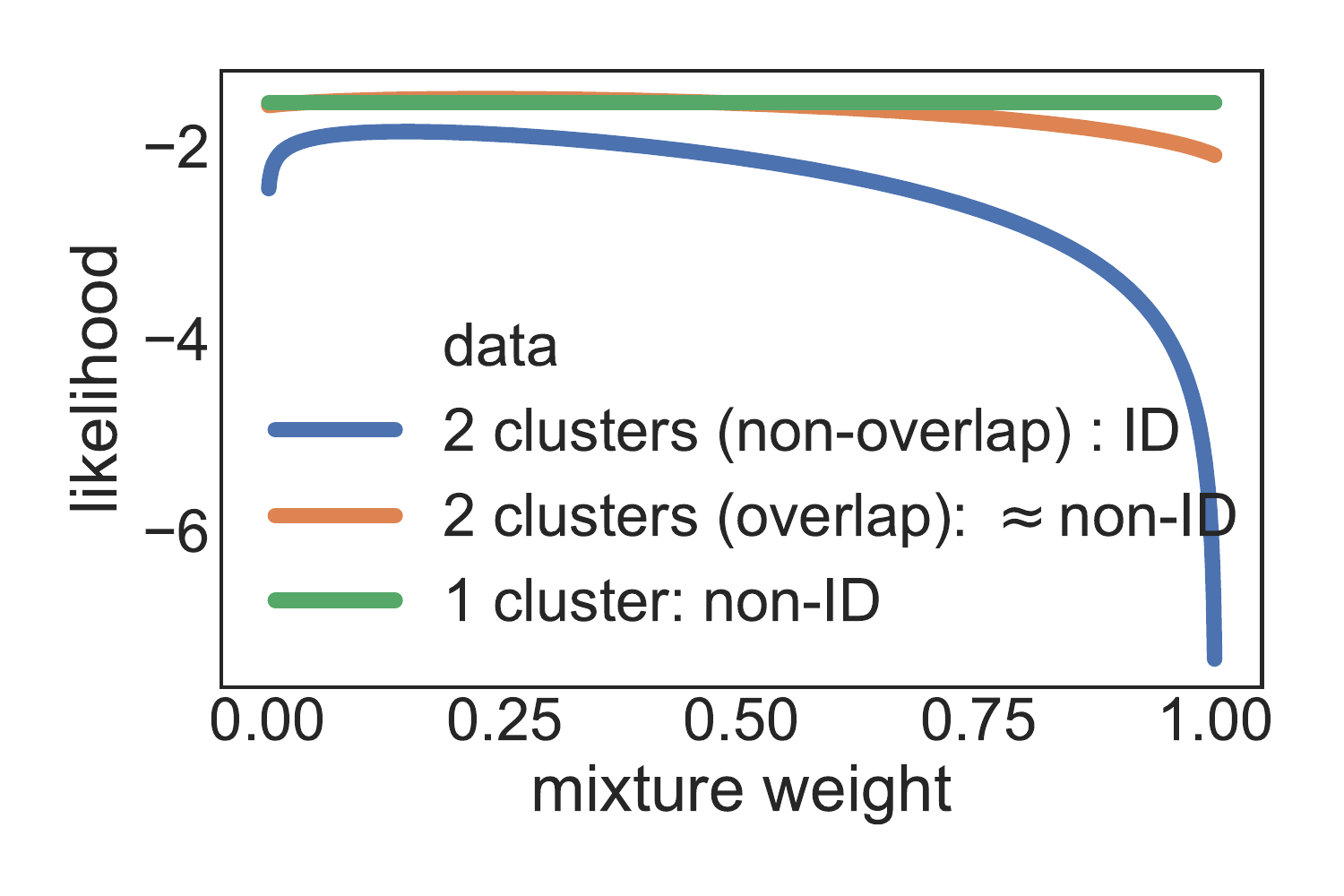}
\caption{Likelihood function \label{fig:gmm_ll}}
\end{subfigure}%
\begin{subfigure}{0.347\textwidth}
\includegraphics[width=\textwidth]{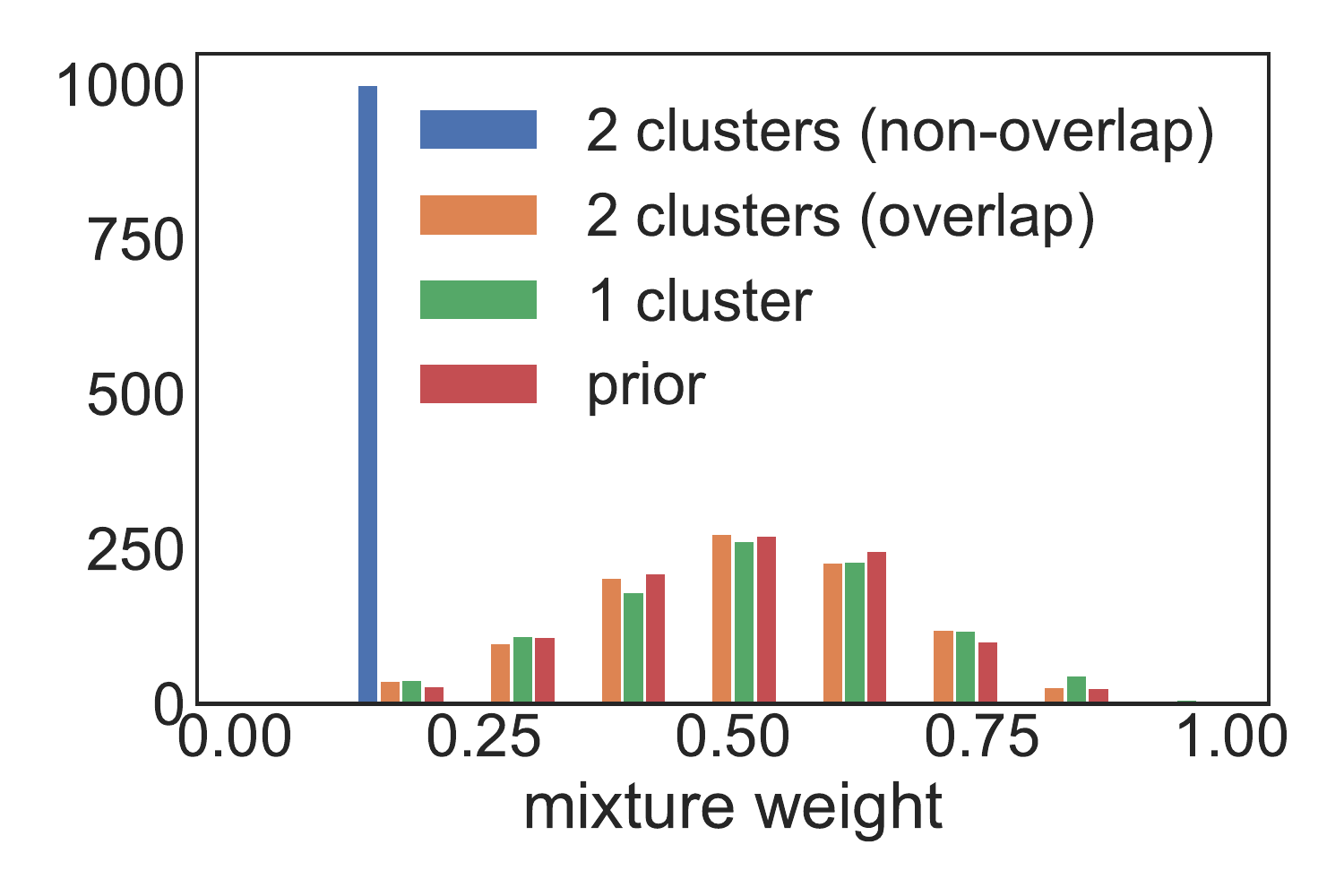}
\caption{Posterior histogram\label{fig:gmm_post}}
\end{subfigure}
\caption{When a latent variable is non-identifiable (non-ID) in a
model, its likelihood function is a constant function and its
posterior is equal to the prior, i.e. its posterior collapses. 
Consider a Gaussian mixture model with two clusters $x\sim
\alpha\cdot \cN(\mu_1,\sigma_1^2)+(1-\alpha) \cdot \cN(\mu_2,
\sigma_2^2)$, treating the mixture weight $\alpha$ as the latent
variable and others as parameters. Fit the model to datasets generated
respectively by one Gaussian cluster ($\alpha$ non-identifiable), two
overlapping Gaussian clusters ($\alpha$ nearly non-identifiable),
and two non-overlapping Gaussian clusters ($\alpha$ identifiable).
Under optimal parameters, the likelihood function $p(x\g \alpha)$ is
(close to) a constant when the latent variable $\alpha$ is (close to)
non-identifiable; its posterior is also (close to) the prior.
Otherwise, the likelihood function is non-constant and the posterior
is peaked.~\label{fig:gmm_collapse}}
\end{figure}

Though we have focused on the posterior collapse of local latent
variables, a model can also suffer from posterior collapse of its
global latent variables. Consider a simple \gls{GMM} with two clusters,
\begin{align*}
  p(\alpha) &= \mathrm{Beta}(\alpha\g 5, 5), \\
  p(x_i\g \alpha\s\theta) &= \alpha \cdot \cN(x_i\g \mu_1, \sigma_1^2) + (1-\alpha)
                            \cdot \cN(x_i \g \mu_2, \sigma_2^2).
\end{align*}
Here $\alpha$ is a global latent variable and
$\theta = (\mu_1, \mu_2, \sigma_1, \sigma_2)$ are the parameters of
the model. Fit this model to three datasets, each with $10^5$ samples.

\begin{enumerate}[leftmargin=*]
\item Sample the data from two non-overlapping clusters,
  \begin{align}
    x_i \sim 0.15\cdot\cN(-10, 1) + 0.85\cdot\cN(10, 1).
  \end{align}
  The latent variable $\alpha$ is identifiable. The two data
  generating clusters are substantially different, so the likelihood
  function varies across $\alpha\in[0,1]$ under the \gls{ML}
  parameters (\Cref{fig:gmm_ll}). The posterior of $\alpha$ is also
  peaked (\Cref{fig:gmm_post}) and differs much from the prior.

\item Sample the data from two overlapping clusters,
  \begin{align}
    x_i\sim 0.15\cdot\cN(-0.5, 1) + 0.85\cdot\cN(0.5, 1).
  \end{align}
  The latent variable $\alpha$ is identifiable. However, it is nearly
  non-identifiable. While the two data generating clusters are
  different, they are very similar to each other because they overlap.
  Therefore, the likelihood function $p(x_i\g \alpha\s \theta^*)$ is
  slowly varying under \gls{ML} parameters
  $\theta^* = (\mu_1^*, \mu_2^*, \sigma_1^*, \sigma_2^*) = (-0.5, 0.5,
  1, 1)$; see \Cref{fig:gmm_ll}.  Consequently, the posterior of
  $\alpha$ remains very close to the prior; see \Cref{fig:gmm_post}.

\item Sample the data from a single Gaussian distribution,
  $x_i \sim \cN(-1, 1)$.  The latent variable $\alpha$ is
  non-identifiable. The reason is that one set of \gls{ML} parameters
  is
  $\theta^* = (\mu_1^*, \mu_2^*, \sigma_1^*, \sigma_2^*) = (-1, -1, 1,
  1)$, i.e. setting both of the two mixture components equal to the
  true data generating Gaussian distribution.

  Under this $\theta^*$, the latent variable $\alpha$ is
  non-identifiable and its likelihood function
  $p(\{x_i\}_{i=1}^n\g \alpha \s \theta^*)$ is constant in $\alpha$
  because the two mixture components are equal; \Cref{fig:gmm_ll}
  illustrates this fact. Moreover, the posterior of $\alpha$
  collapses,
  $p(\alpha \g \{x_i\}_{i=1}^n \s \theta^*) =
  p(\alpha).$~\Cref{fig:gmm_post} illustrates this fact: The \gls{HMC}
  samples of the $\alpha$ posterior closely match those drawn from the
  prior. (Exact inference is intractable in this case, so we use
  \gls{HMC} as a close approximation to exact inference.) This example
  demonstrates the connection between non-identifiability and
  posterior collapse; it also shows that posterior collapse is not
  specific to variational inference but is an issue of the model and
  the data.

\end{enumerate}

As for PPCA, these GMM examples demonstrate that whether a latent
variable is identifiable in a probabilistic model not only depends on
the model but also the data. While all three examples were fitted with
the same \gls{GMM} model, their identifiability situation differs as
the samples are generated in different ways.

\section{Proof of \Cref{prop:IDVAE}}
\label{sec:proof-IDVAE}

We prove a general version of \Cref{prop:IDVAE} by establishing the
latent variable identifiability and flexibility of the most general
form of the \gls{IDVAE}. The \gls{IDVAE}, \gls{IDMVAE}, and
\gls{IDSVAE} (\Cref{defn:idvae,defn:id-mixture-vae,defn:id-svae}) will
all be its special cases. Then \Cref{prop:IDVAE} will also be a
special case of the more general result stated
below~(\Cref{prop:IDVAE-general}).

We first define the most general form of \gls{IDVAE}.

\begin{defn}[General \gls{IDVAE} via Brenier maps]\label{defn:idvae}
A general
\gls{IDVAE} via Brenier maps generates an $D$-dimensional data-point
$x_i, \in\{1,
\ldots, n\}$ by:
\begin{align}
(z_i)_{K\times 1} &\sim p(z_i), \label{eq:id-vae-mt-1}\\
(w_i)_{M\times 1}\g z_i &\sim \mathrm{EF}(w_i\g {\beta_{1}^\top \, z_i }),\label{eq:id-vae-mt-2}\\
(x_i)_{D\times 1}\g w_i, x_{<i} &\sim \mathrm{EF}(x_i\g h\circ
g_{2,\theta}(\beta_{2}^\top \,g_{1,\theta}([w_i, f_\theta(x_{<i})])
)),\label{eq:id-vae-mt-3}
\end{align}
where $\mathrm{EF}$ stands for exponential family distributions; $z_i$
is a $K$-dimensional latent variable, discrete or continuous. The
parameters of the model are $\theta~=~(g_{1,\theta}, g_{2,\theta},
f_\theta)$, where $f_\theta: \mathcal{X}_{<i}\rightarrow
\mathbb{R}^{H}$ is a function that maps all previous data points
$x_{<i}$ to an $H$-dimensional vector, $g_{1,\theta}:
\mathbb{R}^{M+H}\rightarrow \mathbb{R}^{M+H}$ and $g_{2,\theta}:
\mathbb{R}^{D}\rightarrow \mathbb{R}^{D}$ are two continuous monotone
transport maps. The function $h(\cdot)$ is a bijective link function
for the exponential family, e.g. the sigmoid function. The matrix
$\beta_{1}$ is a $K\times M$-dimensional matrix $(M\geq K)$ all the
main diagonal entries being one and all other entries being zero, and
thus with full row rank. Similarly, $\beta_{2}$ is a $(M+H)\times
D$-dimensional matrix $(D \geq M+H)$ with all the main diagonal
entries being one and all other entries being zero, also with full row
rank. Finally, $[w_i, f_\theta(x_{<i})]$ is an $(M+H)\times 1$ vector
that represents a row-stack of the vectors $(w_i)_{M\times 1}$ and
$(f_\theta(x_{<i}))_{H\times 1}$.
\end{defn}

The general \gls{IDVAE} differs from the classical \gls{VAE} whose
general form is
\begin{align}
(z_i)_{K\times 1} &\sim p(z_i), \label{eq:classical-vae-mt-1}\\
(w_i)_{M\times 1}\g z_i &\sim \mathrm{EF}(w_i\g {\beta_{1}^\top \, z_i }),\label{eq:classical-vae-mt-2}\\
(x_i)_{D\times 1}\g w_i, x_{<i} &\sim \mathrm{EF}(x_i\g h\circ
g_\theta([w_i, f_\theta(x_{<i})])),\label{eq:classical-vae-mt-3}
\end{align}
The key difference is in \Cref{eq:classical-vae-mt-3}, where the
classical \gls{VAE} uses an arbitrary function
$g:\mathbb{R}^{M+H}\rightarrow
\mathbb{R}^{D}$ in \Cref{eq:classical-vae-mt-3}. In contrast,
\gls{IDVAE} uses a composition $g_{2,\theta}(\beta_{2}^\top
\,g_{1,\theta}(\cdot))$ with additional constraints in
\Cref{eq:id-vae-mt-3}.

General \gls{IDVAE} can handle both i.i.d. and sequential data. For
i.i.d data (e.g. images), we can set $f_\theta(\cdot)$ to be a zero
function, which implies $P(x_i\g w_i, x_{<i}) = P(x_i\g w_i)$. For
sequential data (e.g. text), we can set $f_\theta(\cdot)$ to be an
LSTM that embeds the history $x_{<i}$ into an $H$-dimensional vector.

General \gls{IDVAE} emulate many existing \gls{VAE}. Letting $z_i$ be
categorical (one-hot) vectors, the distribution $\mathrm{EF}(z_i^\top
\beta_\theta)$ is an exponential family mixture. The identifiable
\gls{VAE} then maps this mixture model through a flexible function
$g_\theta$. When $z_i$ is real-valued, it mimics classical \gls{VAE}
by mapping an exponential family PCA through flexible functions.

\gls{IDGMVAE} is a special case of the general \gls{IDVAE} when we set
 $z_i$ be categorical (one-hot) vectors, set the exponential family
 distribution $\mathrm{EF}$ to be Gaussian in
\Cref{eq:id-vae-mt-2,eq:id-vae-mt-3}. In this case, $w_i\sim
\mathrm{Gaussian}(z_i^\top \beta_\theta, \gamma_\theta)$ is a
Gaussian mixture. Then, we set $f_\theta(\cdot)$ to be a zero
function, which implies $P(x_i\g w_i, x_{<i}) = P(x_i\g w_i)$, and
finally set $h$ as the identity function.

This general \gls{IDVAE} also subsumes the Bernoulli mixture model,
which is a common variant of \gls{IDGMVAE} for the MNIST data.
Specifically, we can set $z_i$ be categorical (one-hot) vectors, and
then set the exponential family distribution $\mathrm{EF}$ to be
Gaussian in \Cref{eq:id-vae-mt-2}, making $w_i\sim
\mathrm{Gaussian}(z_i^\top \beta_\theta, \gamma_\theta)$ to be a
Gaussian mixture. Next we set $f_\theta(\cdot)$ to be a zero function,
which implies $P(x_i\g w_i, x_{<i}) = P(x_i\g w_i)$, then set $h$ to
be the sigmoid function, and finally set the $\mathrm{EF}$ to be
Bernoulli in \Cref{eq:id-vae-mt-3}.

\gls{IDSVAE} is another special case of the general \gls{IDVAE} when
we set the $\mathrm{EF}$ to be a point mass and $\beta_{1,\theta}$ to
be identity matrix in \Cref{eq:id-vae-mt-2}, which implies $w_i=z_i$.
Then setting the $\mathrm{EF}$ to be a categorical distribution and
$h$ to be identity in \Cref{eq:id-vae-mt-3} leads to a configuration
that is the same as \Cref{defn:id-svae}.

\gls{IDVAE} can be made deeper with more layers by introducing
additional full row-rank matrices $\beta_{k}$ (e.g. ones with all the
main diagonal entries being one and all other entries being zero) and
additional Brenier maps $g_{k,\theta}$. For example, we can expand
\Cref{eq:id-vae-mt-3} with an additional layer by setting
\[(x_i)_{D\times 1}\g w_i, x_{<i} \sim
\mathrm{EF}(g_{3,\theta}(\beta_{3}^\top
g_{2,\theta}(\beta_{2}^\top \,g_{1,\theta}([w_i,
f_\theta(x_{<i})])))).\]

Next we establish the latent variable identifiability and flexibility
of this general class of \gls{IDVAE}, which will imply the
identifiability and flexibility of all the special cases above.

\begin{prop} 
\label{prop:IDVAE-general}
The latent variable $z_i$ is identifiable in
\gls{IDVAE}, i.e. for all $i\in\{1, \ldots, n\}$, we have
\begin{align}
p(x_i \g z_i = \tilde{z}', x_{<i}\s \theta) = p(x_i \g z_i =
\tilde{z}, x_{<i} \s \theta) 
\qquad \Rightarrow \qquad\tilde{z}' =
\tilde{z}, \qquad \forall \tilde{z}', \tilde{z}, \theta.
\end{align}
Moreover, for any data distribution generated by the classical
\gls{VAE}
(\Cref{eq:classical-vae-mt-1,eq:classical-vae-mt-2,eq:classical-vae-mt-3}),
there exists an \gls{IDVAE} that can generate the same distribution.
\end{prop}

\begin{proof}
We first establish the latent variable identifiability. To show that
the latent variable $z_i$ is identifiable, it is sufficient to show
that the mapping from $z_i$ to $p(x_i\g z_i \s
\theta)$ is injective for all $\theta$. The injectivity holds because
all the transformations $(\beta_{1},
\beta_{2},g_{1,\theta}, g_{2,\theta})$ involved in the mapping
is injective, and their composition must be injective: the linear
transformations $(\beta_{1}, \beta_{2})$ have full row rank and hence
are injective; the nonlinear transformations $(g_{1,\theta},
g_{2,\theta})$ are monotone transport maps and are guaranteed to be
bijective~\citep{ball2004elementary,mccann2011five}; finally, the
exponential family likelihood is injective.

We next establish the flexibility of the \gls{IDVAE}, by proving that
any \gls{VAE}-generated $p(\mbx)$ can be generated by an
\gls{IDVAE}. The proof proceeds in two steps: (1) we show any
\gls{VAE}-generated $p(\mbx)$ can be generated by a \gls{VAE} with
injective likelihood $p(x_i\g z_i\s \theta)$; (2) we show any
$p(\mbx)$ generated by an injective \gls{VAE} can be generated by an
\gls{IDVAE}.

To prove (1), suppose $\beta_{1}$ does not have full row rank and
$g_\theta$ is not injective. Then there exists some $Z'\in
\mathbb{R}^d$, $d<K$, and injective $\beta'_{1,\theta}, g'_\theta$
such that the new \gls{VAE} can represent the same $p(\mbx)$. The
reason is that we can always turn an non-injective function into an
injective one by considering its quotient space. In particular, we
consider the quotient space with the equivalence relation between
$z,z'$ defined as $p(x\g z\s\theta) = p(x\g z'\s \theta)\}$, which
induces a bijection into $\mathbb{R}^d$. When $p(z')$ is no longer
standard Gaussian, there must exist a bijective Brenier map
$\tilde{z}=f_t(z')$ such that $p(\tilde{x})$ is standard Gaussian
(Theorem 6 of \citet{mccann1995existence}).

To prove (2), we show that any \gls{VAE} with injective mapping can be
reparameterized as a \gls{IDVAE}. To prove this claim, it is
sufficient to show that any injective function
$l_\theta:\mathbb{R}^{M+H}\rightarrow
\mathbb{R}^D$ can be reparametrized as
$g_{2,\theta}(\beta_{2}^\top \,g_{1,\theta}(\cdot))$. Below we provide
such a reparametrization by solving for $g_1, g_2$ and $\beta$ in
$l_\theta(z) = g_{2,\theta}(\beta_2^\top g_{1,\theta}(z))$. We set
$g_{1,\theta}$ as an identity map, $\beta_2$ as an $(M+H)\times D$
matrix with all the main diagonal entries being one and all other
entries being zero, and $g_{2,\theta}$ as an invertible
$\mathbb{R}^d\rightarrow\mathbb{R}^d$ mapping which coincides with
$l_\theta$ on the $(M+H)$-dimensional subspace of $z$.

Finally, we note that the same argument applies to the variant of
\gls{VAE} where $w_i=z_i$. It coincides with the classical \gls{VAE}
in \citet{diederik2014auto}. Applying the same argument as above
establishes \Cref{prop:IDVAE}.

\end{proof}

\section{Experiment details}

\label{sec:expm_details}

For image experiments, all hidden layers of the neural networks have
512 units. We choose the number of continuous latent variables as 64
and the dimensionality of categorical variables as the number of
ground truth labels. Then we use two-layer
RealNVP~(\citep{dinh2016density}) as an approximating family to tease
out the effect of variational inference.

For text experiments, all hidden layers of the neural networks have
1024 units. We choose the dimensionality of the embedding as 1024.
Then we use two-layer LSTM as an approximating family following common
practice of fitting sequential \gls{VAE}.

\section{Additional experimental results}

\Cref{tab:gmvae-1} includes additional experimental results of \gls{IDVAE} on image datasets (Pinwheel and MNIST).


\begin{table*}[t]
\footnotesize
  \begin{center}
    \begin{tabular}{lcccccccccccc} 
     \toprule
      &\multicolumn{4}{c}{Pinwheel}&&&\multicolumn{4}{c}{MNIST}\\
       & \textbf{AU}  & \textbf{KL} & \textbf{MI} & \textbf{LL}&& 
       & \textbf{AU}  & \textbf{KL} & \textbf{MI} & \textbf{LL}\\
          \midrule
          \gls{VAE}~\citep{diederik2014auto}& 0.2 & 1.4e-6 & 2.0e-3 & \bfseries{-6.2} (5e-2)
          &&&0.1 &0.1 &0.2 &-108.2 (5e-1)\\
        SA-\gls{VAE}~\citep{kim2018semi} & 0.2 & 1.6e-5 & 2.0e-2 & -6.5 (5e-2)
        &&&  0.4 &  0.4 & 0.6 & -106.3 (7e-1)\\
        Lagging \gls{VAE}~\citep{he2019lagging} &0.6 & 0.7e-3 & 1.5e0 & -6.5 (4e-2)
        &&& 0.5 & 0.8 &  1.7 & -105.2 (5e-1)\\
        $\beta$-\gls{VAE}~\citep{higgins2016beta} ($\beta$=0.2)&\bfseries{1.0} &\bfseries{1.2e-3} &\bfseries{2.3e0} &-6.6 (6e-2)
        &&&0.8 &1.5 &2.8 &-100.4 (6e-1)\\        
        \gls{IDGMVAE} (this work) &\bfseries{1.0} &\bfseries{1.2e-3} &2.2e0 &-6.5 (5e-2)
        &&&\bfseries{1.0} &\bfseries{1.8} &\bfseries{3.9} &\bfseries{-95.4} (7e-1) \\
\bottomrule
    \end{tabular}
    \vspace{0.2em}
    \caption{
    \gls{IDGMVAE} do not suffer from posterior
    collapse and achieves better fit than its classical counterpart in
    a 9-layer generative model. The reported number is mean (sd) over
    ten different random seeds. (Higher is better.)
    \label{tab:gmvae-1}}
  \end{center}
\end{table*}

\clearpage
\putbib[BIB1]
\end{bibunit}

\end{document}